\documentclass[lettersize,journal]{IEEEtran}
\usepackage{amsmath,amsfonts}
\usepackage{subfigure}
\usepackage{algorithmic}
\usepackage{algorithm}
\usepackage{array}
\usepackage[caption=false,font=normalsize,labelfont=sf,textfont=sf]{subfig}
\usepackage{amsthm}
\usepackage{textcomp}
\usepackage{stfloats}
\usepackage{url}
\usepackage{verbatim}
\usepackage{graphicx}
\usepackage{cite}
\usepackage{booktabs}  
\usepackage{array}    
\usepackage{amssymb} 
\usepackage[table,xcdraw]{xcolor}
\usepackage{multirow} 
\usepackage{makecell} 
\usepackage{colortbl}  
\usepackage{adjustbox} 
\usepackage{graphicx}
\allowdisplaybreaks[2] 
\newtheorem{thm}{\hspace{1em} Theorem}
\newtheorem{lem}{\hspace{1em} Lemma}
\newtheorem{assumption}{\hspace{1em} Assumption}
\newtheorem{remark}{Remark}

\begin{document}

\title{FedHL: Federated Learning for Heterogeneous
Low-Rank Adaptation via Unbiased Aggregation}

\author{
    Zihao Peng, Jiandian Zeng,~\IEEEmembership{Member,~IEEE}, Boyuan Li,~\IEEEmembership{Student Member,~IEEE}, Guo Li,~\IEEEmembership{Student Member,~IEEE}, Shengbo Chen,~\IEEEmembership{Member,~IEEE}, Tian Wang*,~\IEEEmembership{Senior Member,~IEEE}%
    \thanks{This work was supported in part by grants from the National Natural Science Foundation of China (NSFC) (62172046, 62372047, 62302049), the Beijing Natural Science Foundation (No. 4232028), the Natural Science Foundation of Guangdong Province (2024A1515011323), Zhuhai Basic and Applied Basic Research Foundation (2220004002619), the Joint Project of Production, Teaching and Research of Zhuhai (2220004002686, 2320004002812), Science and Technology Projects of Social Development in Zhuhai (2320004000213), the Supplemental Funds for Major Scientific Research Projects of Beijing Normal University, Zhuhai (ZHPT2023002), the Fundamental Research Funds for the Central Universities, Higher Education Research Topics of Guangdong Association of Higher Education in the 14th Five-Year Plan under 24GYB207, Beijing Normal University Education Reform Project under jx2024139, and support from the Interdisciplinary Intelligence Super Computer Center of Beijing Normal University at Zhuhai. (\textit{Corresponding author}: Tian Wang).}
    \thanks{Zihao Peng, Jiandian Zeng, and Guo Li are with the Institute of Artificial Intelligence and Future Networks, Beijing Normal University, Zhuhai 519087, China (e-mail: \{pzh\_cs, liguo\}@mail.bnu.edu.cn; jiandian@bnu.edu.cn).}
    \thanks{Boyuan Li is with the School of Computer Science and Artificial Intelligence, Zhengzhou University, Zhengzhou 450001, China (e-mail: l202311841010602@gs.zzu.edu.cn).}
    \thanks{Shengbo Chen is with the School of Software, Nanchang University, Nanchang 330000, China (e-mail: ccb02kingdom@gmail.com).}
    \thanks{Tian Wang is with the Institute of Artificial Intelligence and Future Networks, Beijing Normal University, Zhuhai 519087, China (e-mail: tianwang@bnu.edu.cn).}
}
\maketitle

\begin{abstract}
Federated Learning (FL) facilitates the fine-tuning of Foundation Models (FMs) using distributed data sources, with Low-Rank Adaptation (LoRA) gaining popularity due to its low communication costs and strong performance. While recent work acknowledges the benefits of heterogeneous LoRA in FL and introduces flexible algorithms to support its implementation, our theoretical analysis reveals a critical gap: existing methods lack formal convergence guarantees due to parameter truncation and biased gradient updates.  Specifically, adapting client-specific LoRA ranks necessitates truncating global parameters, which introduces inherent truncation errors and leads to subsequent inaccurate gradient updates that accumulate over training rounds, ultimately degrading performance. To address the above issues, we propose \textbf{FedHL}, a simple yet effective \textbf{Fed}erated Learning framework tailored for \textbf{H}eterogeneous \textbf{L}oRA. By leveraging the full-rank global model as a calibrated aggregation basis, FedHL eliminates the direct truncation bias from initial alignment with client-specific ranks. Furthermore, we derive the theoretically optimal aggregation weights by minimizing the gradient drift term in the convergence upper bound. Our analysis shows that FedHL guarantees $\mathcal{O}(1/\sqrt{T})$ convergence rate, and experiments on multiple real-world datasets demonstrate a 1–3\% improvement over several state-of-the-art methods.
\end{abstract}


\section{Introduction}
\IEEEPARstart{R}{apid} advances in Artificial Intelligence (AI) are increasingly fueled by Foundation Models (FMs)—large, pre-trained models that excel at a wide range of downstream tasks due to their extensive parameterization and broad knowledge \cite{zhuang2023foundation,ren2024advances}. Federated Learning (FL) \cite{mcmahan2017communication}, a distributed paradigm that trains models collaboratively without sharing raw data, has garnered growing interest. 
Integrating FMs and FL offers mutual benefits: on the one hand, FL harnesses data from diverse, heterogeneous sources \cite{li2023review}, thus enriching the training data accessible to FMs; on the other hand, the strong generalization capability of FMs can serve as a powerful initializer or teacher, boosting FL’s overall performance \cite{fan2025ten,nguyen2022begin}.

However, integrating FMs with FL presents significant challenges due to communication bottlenecks \cite{guo2023promptfl,wu2024fedbiot}. In traditional FL, the transmission of full model parameters during each training round imposes prohibitively high communication costs. To address this issue, Parameter-Efficient Fine-Tuning (PEFT) techniques have emerged as viable alternatives. Among them, Low-Rank Adaptation (LoRA) \cite{hu2021lora, zhang2023adalora} stands out for its ability to fine-tune large models by updating a small fraction of parameters while achieving comparable performance to full fine-tuning. By exploiting the inherent low-rank structure in parameter updates, it decomposes weight changes into two low-rank matrices and transfers only these matrices, significantly reducing the communication overhead associated with fine-tuning large models in FL \cite{jiang2024low,10916512}.

\begin{figure}[!t]
\centering
\includegraphics[width=\columnwidth, height=5.3cm]{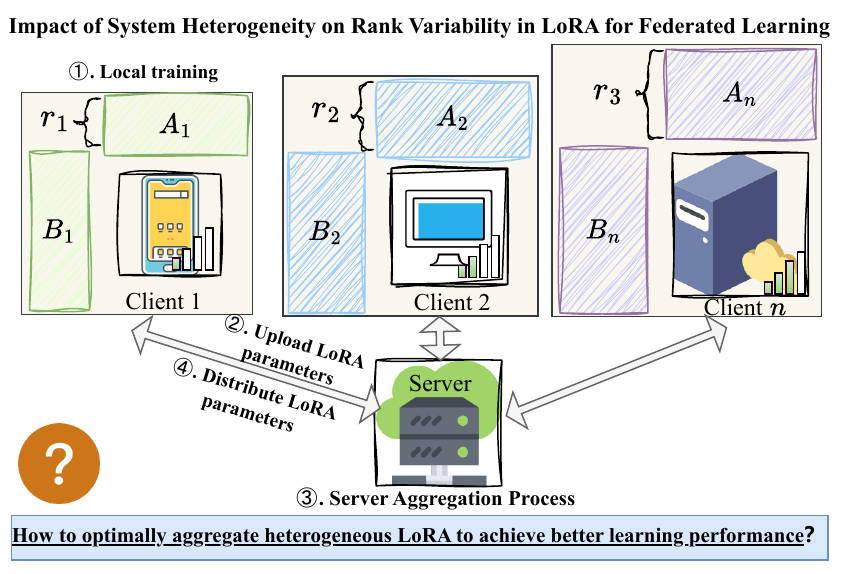}
\caption{Heterogeneous client environments necessitate LoRA modules of varying ranks, requiring the design of targeted federated aggregation algorithms.}
\label{fig:fedlora_motivation}
\end{figure}

Building on the communication efficiency of LoRA, researchers have extended its application to address another pressing concern in practical FL scenarios: system and data heterogeneity \cite{zhang2024towards}. In real-world settings, clients vary significantly in terms of computational power, network conditions, and data distributions, making uniform LoRA configurations less compatible with the complexities of heterogeneous environments \cite{cho2023heterogeneous,10666083}. To enhance LoRA's applicability in such scenarios, researchers have developed heterogeneous LoRA methods, which enable customized low-rank configurations for each client, aligning parameter updates with local resource constraints while preserving LoRA’s core advantage of reduced communication \cite{cho2024heterogeneous}.  

Although heterogeneous LoRA has shown promise in certain studies \cite{bai2024federated,cho2024heterogeneous}, its convergence problems have yet to be systematically addressed. A critical issue arises during model initialization, where clients are required to truncate the global model to comply with their rank-specific constraints \cite{fang2025federated}. Specifically, adapting the rank of individual clients’ low-rank matrices necessitates truncating the global model’s parameters, which inherently introduces truncation bias \cite{bai2024federated}. Furthermore, performing gradient descent from such a truncated starting point results in deviations from the ideal gradient, which accumulate over iterations and lead to a biased update trajectory. Despite the significance of this issue, existing approaches largely overlook the impact of truncation-induced biases. Therefore, the design of effective training and aggregation strategies for heterogeneous LoRA modules remains an open problem, particularly in addressing truncation-induced biases and ensuring stable convergence in federated settings.

In this paper, we provide a rigorous theoretical analysis that identifies a critical limitation in existing FL frameworks utilizing heterogeneous LoRA: the accumulation of truncation-induced bias over training epochs, which ultimately undermines convergence guarantees.  To address these challenges, we propose FedHL, a \textbf{Fed}erated Learning framework tailored for \textbf{H}eterogeneous \textbf{L}oRA. FedHL  integrates the full-rank model into the aggregation process, effectively eliminating truncation bias and the resulting distortion terms that impede convergence. Moreover, our framework derives theoretically optimal aggregation weights by optimizing the convergence bound, thereby mitigating gradient drift. We theoretically prove that FedHL achieves a convergence rate of \(\mathcal{O}(1/\sqrt{T})\), demonstrating that it converges to the optimal solution. The main contributions of this work are as follows: 
\begin{itemize} 
\item 
We conduct a rigorous analysis of the convergence behavior of heterogeneous LoRA in FL, revealing that the commonly used truncation operation to align client-specific LoRA ranks introduces biases that undermine robust convergence guarantees.
\item We propose FedHL, a framework that leverages the full-rank model for unbiased aggregation to eliminate initialization truncation bias and derives optimal aggregation weights to mitigate gradient drift during training,  ensuring convergence to the optimal solution.
\item We validate the effectiveness of FedHL through extensive experiments on diverse federated fine-tuning benchmarks, and the experimental results demonstrate its superiority over the state-of-the-art methods. 
\end{itemize}

The remainder of this paper is organized as follows. Section \ref{related_work} provides an overview of related work. Section~\ref{preliminaries} introduces the necessary preliminaries. Section \ref{section:FedHL} discusses the challenges in existing FL-based LoRA approaches and presents the \textsc{FedHL} algorithm in detail. The convergence analysis is provided in Section \ref{section:convergence}. Experimental results are reported in Section \ref{section:experimental results}. Finally, Section \ref{sec:conclusion} concludes the paper. Table \ref{tab:notation} presents the symbols used in this paper.

\begin{table}[!t]
  \centering
  \caption{Primary Notation Summary.}
  \label{tab:notation}
  \footnotesize
  \setlength{\tabcolsep}{4pt}
  \renewcommand{\arraystretch}{1.1}
  \begin{tabular}{@{}cc@{}}
    \toprule 
    \textbf{Symbol} & \textbf{Definition} \\ 
    \hline
    $W$               & Global parameter matrix, $W \in \mathbb{R}^{d \times k}$ \\ 
    $r_i$             & Rank allocated to client $i$ in LoRA \\ 
    $W_t$             & Global model at communication round $t$ \\ 
    $W_t^{r_i}$       & Rank-$r_i$ truncated SVD projection of $W_t$ \\ 
    $B_t^i$           & Low-rank matrix of client $i$, $B_t^i \in \mathbb{R}^{d \times r_i}$ at round $t$ \\ 
    $A_t^i$           & Low-rank matrix of client $i$, $A_t^i \in \mathbb{R}^{r_i \times k}$ at round $t$ \\ 
    $B_{t,K}^i$       & Updated $B_t^i$ after $K$ local training epochs \\ 
    $A_{t,K}^i$       & Updated $A_t^i$ after $K$ local training epochs \\ 
    $W_{t+1}^i$       & Reconstructed model for client $i$ at $t{+}1$, i.e., $W_{t+1}^i = B_{t,K}^i A_{t,K}^i$ \\ 
    $p_i(t)$          & Aggregation weight of client $i$ at round $t$ \\ 
    $\hat r_i(t)$     & Truncation error, defined as $\| W_t - W_t^{r_i} \|_F^2$ \\ 
    $\eta_t$          & Local learning rate at round $t$ \\ 
    $K$               & Number of local epochs per communication round \\ 
    $T$               & Total number of communication rounds \\ 
\bottomrule 
  \end{tabular}
\end{table}

\section{Related Works}
\label{related_work}
\subsection{Low-Rank Adaptation of Large Language Models}
LoRA \cite{hu2021lora} is one of the most advanced parameter-efficient fine-tuning methods, motivated by the observation that weight updates in neural networks exhibit low-rank properties. Specifically, LoRA re‑parameterizes the weight‑matrix update as the product of two low‑rank matrices, thereby emulating the full update while greatly reducing the number of trainable parameters. Owing to its plug-and-play nature and no additional inference overhead, LoRA has been widely adopted to adapt large models across diverse downstream domains, including vision~\cite{guo2024smooth}, language~\cite{xue2024autore}, and multimodal tasks~\cite{ye2023mplug}.

Several studies further optimize LoRA-based fine-tuning approaches to enhance model efficiency and adaptability. QLoRA \cite{dettmers2024qlora} introduces parameter quantization techniques, reducing memory usage while maintaining performance. AdaLoRA \cite{zhang2023adaptive} improves model performance by assigning ranks based on the importance of each model layer’s weights. LoRA+ \cite{hayou2024lora} demonstrates that using different learning rates for the two low-rank modules yields better results. LoRAHub \cite{huang2023lorahub} facilitates the purposeful assembly of LoRA modules trained for various tasks, aiming to achieve adaptive performance on unseen tasks. Additionally, the theoretical understanding of LoRA’s interpretability continues to evolve. Malladi et al. \cite{malladi2023kernel} show through kernel analysis that LoRA fine-tuning achieves near-full fine-tuning performance under lazy training mechanisms. Zeng et al. \cite{zeng2024the} analyze the expressiveness of LoRA in fully connected neural networks and transformer architectures, while Jang et al. \cite{jang2024lora} investigate strategies for rank selection to avoid spurious local minima.

\subsection{LoRA for Federated Learning}

As a parameter-efficient fine-tuning technique, LoRA significantly cuts communication costs in FL while enabling task adaptation under memory constraints. Its integration with FL has attracted significant attention \cite{10733964, mao2024survey,jiang2023low}, with early studies focusing on homogeneous setups. For instance, FedIT \cite{zhang2024towards} utilizes homogeneous LoRA for instruction tuning, balancing privacy and efficiency. FDLoRA \cite{qi2024fdlora} employs dual homogeneous modules to integrate global and personalized knowledge, while HyperFLoRA \cite{lu2024hyperflora} introduces a hypernetwork-based approach to generate personalized adapter weights using minimal client statistics. In \cite{sun2024improving, liu2025differentially}, differential privacy is incorporated into homogeneous LoRA, thereby enhancing its practical applicability in privacy-preserving FL. Moreover, to further mitigate communication bottlenecks, Kuo et al. \cite{kuo2024federated} apply sparsity to homogeneous LoRA during communication. Recent studies \cite{wang2025federated,10763424} have proposed integrating LoRA into wireless network scenarios to achieve efficient training.

While homogeneous approaches lay a foundation, practical FL systems often face diverse resource constraints, thus necessitating the exploration of heterogeneous solutions. HetLoRA \cite{cho2023heterogeneous} uncovered a trade-off between overfitting and slow convergence in uniform LoRA rank configurations. To address this, Cho et al. \cite{cho2024heterogeneous} proposed heterogeneous LoRA ranks, combining high- and low-rank benefits through zero-padding and truncation. Similarly, Fan et al. \cite{fan2025helora} also explore padding methods to optimize resource usage in FL systems. FlexLoRA \cite{bai2024federated} intentionally bridges resource disparities between clients via heterogeneous training, alleviating the "bottleneck effect" of traditional FL. FLoRA \cite{wang2024flora} introduces a stacking-based aggregation method for noise-free federated fine-tuning of heterogeneous LoRA adapters.  

Despite its strong empirical promise in FL, heterogeneous LoRA remains theoretically underexplored. Existing convergence analyses overlook the bias introduced by rank reduction, so their guarantees are not formally complete. This paper pinpoints the truncation-induced bias, derives explicit convergence bounds, and proposes an optimization algorithm that provably mitigates it. \textit{To the best of our knowledge, this is the first rigorous theoretical framework for heterogeneous LoRA in federated learning}.

\section{Preliminaries}
\label{preliminaries}
\subsection{Low-Rank Adaptation (LoRA)}
LoRA is a technique that reduces the number of trainable parameters during the fine-tuning of large models \cite{hu2021lora}. Its core idea is to inject two low-rank matrices into the model while keeping the original weights \( \Theta_0 \) frozen. Specifically, if the pre-trained weight matrix \( \Theta_0 \) has dimensions \( d \times k \), LoRA represents the parameter update \( \Delta \Theta \) as the product of two low-rank matrices \( B \) and \( A \). Thus, the fine-tuned weight matrix is given by
\begin{align}
\Theta = \Theta_0 + \Delta \Theta = \Theta_0 + BA.
\end{align}

Here, \( \Delta \Theta \in \mathbb{R}^{d \times k}\), \( B \in \mathbb{R}^{d \times r}\), and \( A \in \mathbb{R}^{r \times k}\), where \( r \ll \min(d, k) \). During fine-tuning, only the low-rank matrices \( B \) and \( A \) are optimized, while \( \Theta_0 \) remains fixed. By doing so, LoRA significantly reduces the number of trainable parameters, making the fine-tuning of large models more feasible in resource-constrained settings. Moreover, the product \( BA \) is scaled by a factor of \( \alpha / r \) to further control the update magnitude.

\subsection{Federated Learning}
Federated learning is a distributed machine learning paradigm where a central server aggregates gradients or model updates sent by multiple clients over several rounds to obtain a global model. For an FL system, we assume that client $i$ owns a local dataset $\mathcal{D}_i$, which contains $m_i$ samples participating in the FL training process. The corresponding loss function of client $i$ is defined as $f_i(W) = \frac{1}{m_i} \sum_{\xi \in \mathcal{D}_i} \ell(W; \xi)$, where $W \in \mathbb{R}^d$ represents the machine learning model to be optimized, and $\ell(W; \xi)$ is the loss evaluated on data sample $\xi$ with model $W$. Assuming there are $N$ clients in total, the overall training objective can be expressed as the weighted sum of the individual objectives of all clients:
\begin{equation}
\min_{W \in \mathbb{R}^d} f(W) = \min_{W \in \mathbb{R}^d} \sum_{i=1}^{N} p_i f_i(W),
\end{equation}
where $p_i = \frac{m_i}{m}$ represents the proportion of client $i$'s local dataset size $m_i$ to the total dataset size $m$.  The ultimate goal is to find the optimal solution $W^*$ of the global objective function $f(W)$. FedAvg \cite{mcmahan2017communication} addresses this problem by iteratively training a global model. In each round, the server distributes the latest model to the clients, which then perform several steps of stochastic gradient descent (SGD) locally and subsequently send back their model updates to the server for aggregation. This process continues for \( T \) rounds, resulting in the final global model.

\section{Federated Learning with Heterogeneous LoRA}
\label{section:FedHL}
\subsection{Aggregation Noise in LoRA under Traditional FL}

In conventional FL, the server aggregates model updates from clients by averaging their weights, assuming uniform dimensionality across all LoRA modules. Consider two clients, each associated with low-rank matrices \( B^1, A^1 \) and \( B^2 , A^2 \) after local training. Ideally, the global model parameter update, denoted as \( W \) (where \(  W \) corresponds to the aggregated clients' LoRA bypass updates), is given by:
\begin{align}
W  = \frac{1}{2} (B^{1} A^{1} + B^{2} A^{2}).
\end{align}
However, traditional FL aggregates the matrices separately, resulting in
\( \hat{B} = \frac{1}{2} (B^{1} + B^{2}) \) and \( \hat{A} = \frac{1}{2} (A^{1} + A^{2}) \). The actual aggregated update, given by \( \hat{B} \hat{A} \), becomes:
\begin{align} \hat{W} &= \frac{1}{4}(B^{1} A^{1} + B^{1}A^{2} + B^{2}A^{1}+ B^{2}A^{2}).
\end{align}
Clearly, \( \hat{W} \neq W \), resulting in unintended cross terms, such as \( B^{2} A^{1} \) and \( B^{1} A^{2} \), which contribute to aggregation noise and make the global LoRA parameter update undesirable.

Recently, several studies \cite{wang2024flora, bai2024federated} have sought to address aggregation noise in FL by recomposing a complete entity 
$W$ for aggregation. While these methods effectively reduce homogeneous aggregation noise at a surface level, they fail to address the complexities of heterogeneous LoRA scenarios, particularly those involving truncation. In this paper, we reconsider this issue and propose a targeted solution for heterogeneous LoRA, aiming to achieve two key goals: \textbf{(1) eliminate LoRA aggregation noise in traditional FL, and (2) further mitigate the bias inherent in the truncation process of heterogeneous LoRA.} 


\subsection{Truncation-induced Biases in Heterogeneous LoRA Re-Ranking}
In heterogeneous LoRA environments, each client uses LoRA modules with personalized rank settings, allowing them to adjust the rank of their LoRA modules according to their individual requirements. Once the updates from all clients are aggregated, the server obtains the global weight \( W \in \mathbb{R}^{d \times k} \). However, to respect each client’s rank constraint \( r_i \), the global weight must then be truncated accordingly for local training.

To achieve this, Singular Value Decomposition (SVD) is commonly employed, as described in \cite{bai2024federated}. In SVD, the global weight \( W \) is factorized into three matrices: \( W = U \Sigma V^\top \), where \( U \) and \( V \) are orthogonal matrices, and \( \Sigma \) is a diagonal matrix containing the singular values \( \sigma_1, \sigma_2, \dots, \sigma_d \). For each client \( i \), we select the top \( r_i \) singular values from \( \Sigma \), forming a truncated version \( \Sigma^{r_i} \) along with their corresponding singular vectors \( U^{r_i} \) and \( (V^\top)^{r_i} \). The global weight is then reconstructed as the low-rank approximation \( W^{r_i} = U^{r_i} \Sigma^{r_i} ({V^\top})^{r_i} \), which is then used to set the LoRA modules of client \( i \). Specifically, the LoRA modules are defined as \( B^{i} = U^{r_i} \sqrt{\Sigma^{r_i}} \) and \( A^{i} = \sqrt{\Sigma^{r_i}} ({V^{\top}})^{r_i} \). 
While this customization allows for greater efficiency in FL, it also introduces several challenges. A key issue arises from the truncation process, which can create discrepancies between the truncated model \( W^{r_i} \) that client \( i \) receives and the global full-rank model \( W \), which is aggregated initially. This misalignment introduces two significant sources of deviation in the FL training process:
\begin{itemize}
    \item \noindent\textbf{Model Truncation Bias.} Truncation causes the initial parameters of each client's model to \textit{diverge from the globally aggregated model} for the next rounds of training. Through rigorous convergence analysis, we prove that these discrepancies consistently accumulate over training rounds, resulting in insufficient convergence guarantees.
    \item \noindent\textbf{Gradient Descent Drift.} By performing gradient descent from a \textit{truncated and misaligned} starting point, the client's local updates deviate from the trajectory that would have been produced by using the ideal full-rank model as the starting point, leading to the propagation of inconsistencies during local training and ultimately degrading the model's performance.
\end{itemize}

The above biases highlight the challenges inherent in heterogeneous LoRA environments. Our convergence analysis indicates that these discrepancies can impede convergence and degrade overall performance if not properly addressed. It is worth noting that while the truncation bias can be mitigated by refreshing the server model in subsequent rounds, the gradient drift—stemming from client resource constraints—cannot be fully eliminated and can only be partially mitigated.

\subsection{The proposed FedHL}

To address aggregation noise in traditional FL and the truncation-induced bias caused by rank redistribution in heterogeneous LoRA, we propose FedHL. Based on rigorous convergence analysis, the proposed framework ensures provable convergence under heterogeneous LoRA conditions. 

\subsubsection{\textbf{Treating the Low-Rank Matrix $W$ as a Unified Entity}}
In traditional FL, the server separately aggregating the low-rank matrices \( B \) and \( A \)  introduces undesired cross-terms, resulting in aggregation noise that degrades model performance. This issue is further exacerbated in heterogeneous environments where clients have varying rank settings, resulting in misaligned updates during aggregation. 

To address this issue, we aggregate low-rank matrices (e.g., \( W = BA \) ) as a single unified entity. Specifically, in the $t$-th round of FL training, each client $i$ trains the low-rank matrices \( B_{t}^i \) and \( A_{t}^i \) locally for \( K \) steps. The updated matrices, \( B_{t,K}^i \)  and \( A_{t,K}^i \), are then transmitted to the server,  which reconstructs the full matrix as \( W_{t+1}^i = B_{t,K}^i A_{t,K}^i \). The aggregation is performed as:
\begin{align}
\label{fl:agg}
W_{t+1} = \sum_{i=1}^N p_i W_{t+1}^i.
\end{align}
By aggregating \( W_{t+1}^i \) as a whole, this approach eliminates cross-term noise caused by the separate aggregation of \( B_{t,K}^i \) and \( A_{t,K}^i \). Furthermore, it accommodates heterogeneous LoRA configurations, allowing each client $i$ to use different rank settings \( r_i \), where \( B_{t,K}^i \in \mathbb{R}^{d \times r_i} \) and \( A_{t,K}^i \in \mathbb{R}^{r_i \times k} \).

\begin{algorithm}[t]
\caption{The FedHL Algorithm.}
\label{alg:algorithm_lora_final}
\begin{algorithmic}[1]
\renewcommand{\algorithmicrequire}{\textbf{Input:}}     
\renewcommand{\algorithmicensure}{\textbf{Output:}}  
\REQUIRE Initial Global Model \( W_0 \in \mathbb{R}^{d \times k} \), Learning Rate \( \eta_t \), Iterations \( T \), Local Steps \( K \), LoRA Ranks \( \{r_i\}_{i=1}^N \)
\ENSURE Final Global Model \( W_{T+1} \in \mathbb{R}^{d \times k} \) and Client-Specific LoRA Parameters \( \{(B_{T+1}^i, A_{T+1}^i)\}_{i=1}^N \)

\FOR{$t = 0$ \textbf{to} $T-1$}
    \STATE \textbf{Model Initialization:}
    \FOR{each client \(i \in [N]\)}
        \STATE Truncate \( W_t \) via SVD to obtain LoRA parameters: \( (B_t^i, A_t^i) \gets \text{SVD}(W_t, r_i) \);  \COMMENT{Perform SVD once for all clients}
        \STATE Record truncation error: \( \hat{r}_i(t) \gets \| W_t - W_t^{r_i} \|^2 \); \COMMENT{Where \( W_t^{r_i} = B_t^i A_t^i \)}
    \ENDFOR
    \STATE \textbf{Client Updates:}
    \FOR{each client \(i \in [N]\)}
        \STATE Receive LoRA parameters: \( B_{t,0}^i \gets B_t^i \), \( A_{t,0}^i \gets A_t^i \);
        \STATE Perform local updates for \( K \) steps: 
        \STATE \quad Obtain \( B_{t,K}^i \), \( A_{t,K}^i \) after \( K \) steps with learning rate \( \eta_t \), then upload them to the server;
    \ENDFOR
    \STATE \textbf{Server Aggregation:}
    \FOR{each client \(i \in [N]\)}
        \STATE Reconstruct local model: \( W_{t+1}^i \gets B_{t,K}^i A_{t,K}^i \);
    \ENDFOR
    \STATE  \fbox{$W_{t+1} \gets \sum_{i=1}^N \mathbf{p}_i(t) \left[ W_t + \left( W_{t+1}^i - W_t^{r_i} \right) \right]$};  \COMMENT{Where \( \mathbf{p}_i(t) \) is computed from Eqn. (\ref{eq:weight}) using \( \hat{r}_i(t) \)}
\ENDFOR
\end{algorithmic}
\end{algorithm}

\subsubsection{\textbf{Utilizing the Complete Global Model $W_t$ as the Aggregation Baseline}}

In each training round \( t \),  client \( i \) initializes its model using a truncated version of the previous global model \( W_t \), which is approximated as \( W_t^{r_i} \) in accordance with the heterogeneous LoRA setup. Specifically, client \( i \) constructs client-specific matrices \( B_t^i \in \mathbb{R}^{d \times r_i} \) and \( A_t^i \in \mathbb{R}^{r_i \times k} \) by performing SVD on \( W_t \), such that \( W_t^{r_i} = B_t^i A_t^i \) represents the rank-\( r_i \) approximation of the global weight matrix. This truncation introduces inherent inconsistencies in the model initialization across clients, as each client employs a distinct rank approximation of the global weight matrix, even prior to the commencement of training. To address this issue, a compensation mechanism is necessary to eliminate the inconsistencies resulting from truncation.

We now investigate how truncation influences the aggregation process. In the current FedLoRA method \cite{bai2024federated, wang2024flora}, after local training, the server aggregates the updates from all clients to obtain the global model for round \( t+1 \), as described in Eqn.~(\ref{fl:agg}) in an incremental update formulation as follows:
\begin{align}
W_{t+1}  &\equiv \sum_{i=1}^N p_i W_{t+1}^i \equiv \sum_{i=1}^N p_i \left( W_t^{r_i} + \Delta W_t^i \right),
\end{align}

\begin{figure*}[htbp]
    \centering
\includegraphics[width=0.90\textwidth, height=10.1cm]{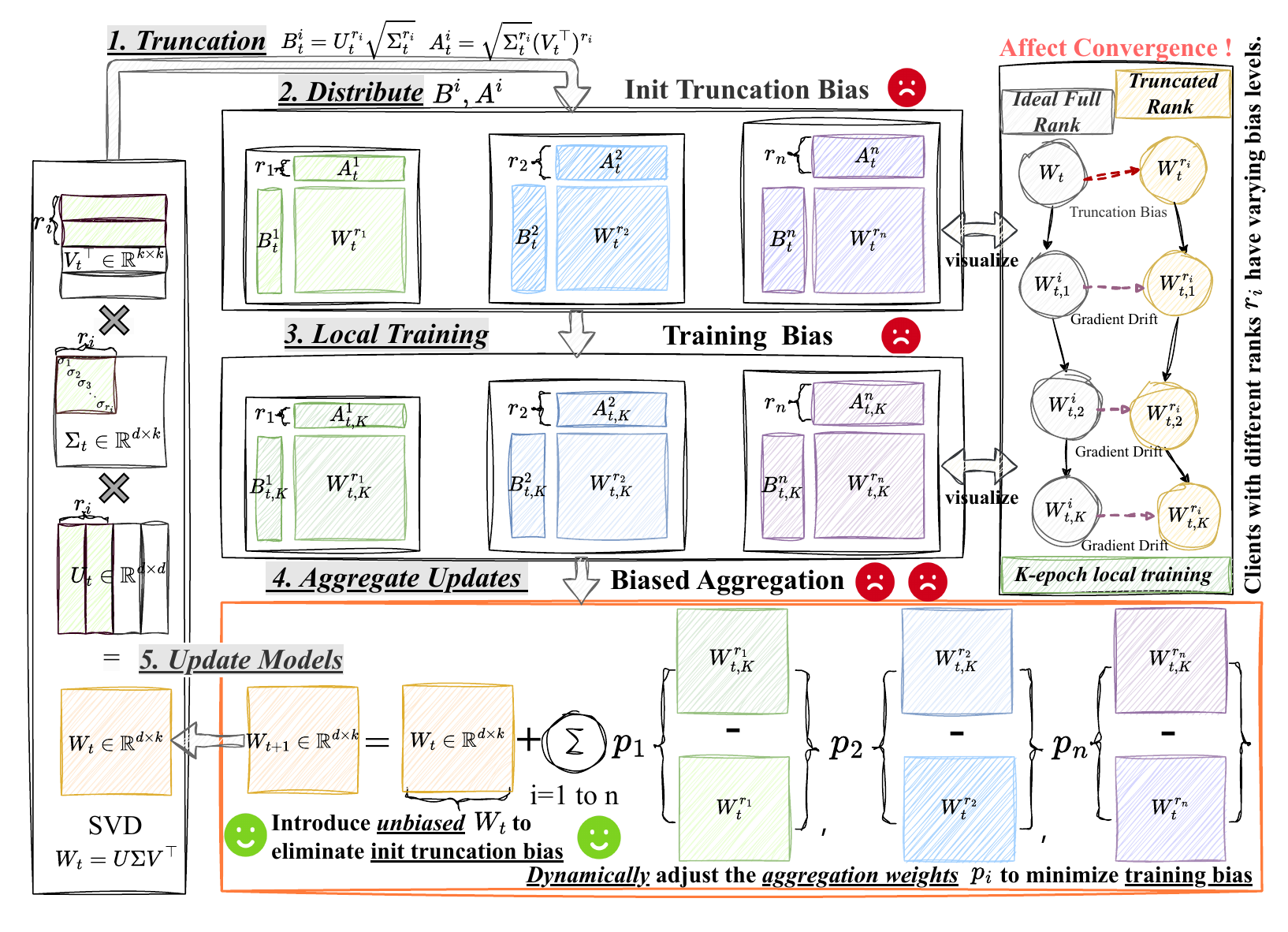}
    \caption{Overview of FedHL. In heterogeneous LoRA, compared to ideal full-rank updates, truncation and gradient biases emerge, as illustrated in the upper-right panel. During server-side aggregation, a closer unbiased aggregation baseline $W_t$ is introduced to correct init truncation biases, and weights $p_i$ are dynamically adjusted to minimize training bias.}
    \label{fig:fedlora}
\end{figure*}
\noindent where \( \Delta W_t^i = W_{t+1}^i - W_t^{r_i} \) represents the local update based on the truncated model \( W_t^{r_i} \) used by client \( i \). In this approach, \( W_t^{r_i} \) serves as the baseline for aggregating local updates. However, this aggregation introduces truncation errors since \(W_t^{r_i}\) is a rank \(r_i\) approximation of the global model \(W_t\), which leads to inconsistencies between models across clients. Our theoretical analysis in Section \ref{sec:convergence analysis} shows that the current aggregation scheme introduces inherent truncation biases that accumulate over multiple aggregation rounds, perturbing the training process and ultimately hindering convergence.

To address the limitations associated with model truncation, we propose utilizing the global full-rank model \( W_t \), stored on the server and aggregated in the preceding round, as the aggregation baseline. By using \( W_t \) as a consistent reference point, the FedHL aggregation rule is formulated as:
\begin{align}
\label{agg}
W_{t+1} = \sum_{i=1}^N p_i \left( W_t + \Delta W_t^i \right).
\end{align}
In this framework, \( W_t \) serves as a stable and unbiased reference point, eliminating the initialization bias introduced by the truncated model \( W_t^{r_i} \). By adopting a global optimization perspective, FedHL first identifies intrinsic aggregation flaws and subsequently addresses them by establishing a unified aggregation baseline. This strategy promotes more accurate learning trajectories, ultimately ensuring robust convergence.

\subsubsection{\textbf{Dynamically adjust the aggregation weight $p_i$ to alleviate gradient drift}}
Although the aggregation baseline \(W_t\) aims to eliminate initialization truncation bias, the local gradient descent process remains constrained by the rank limitations of individual clients. Specifically, clients are required to initialize gradient calculations using the truncated model $W_t^{r_i}$. This implies that gradients computed from \(W_t^{r_i}\) deviate from those derived using the full-rank model, and these deviations influence subsequent gradient computations in the next round. As a result, errors introduced during earlier gradient calculations propagate and amplify through multiple rounds of gradient descent. The global model aggregates gradients influenced by these biases, which may lead to suboptimal model updates. Fortunately, the aggregation weights can be adjusted to counteract the impact of biased gradients. For instance, assigning smaller aggregation weights to clients with larger gradient biases helps reduce their influence. The detailed closed-form solution for the aggregation weights and its proof is provided in Theorem~\ref{theorem:3}.

\subsubsection{\textbf{The Workflow of the FedHL Algorithm}}

The procedure of the FedHL algorithm is formally presented in Algorithm \ref{alg:algorithm_lora_final}. At the beginning of the FL training for round $t \in [ T ] \triangleq \{ 0,1 , \cdots , T-1 \}$, the server  obtains the global model \( W_{t} \in \mathbb{R}^{d \times k} \). For each heterogeneous LoRA client \( i \), the model is initialized using a rank-\( r_i \) approximation derived from the SVD of \( W_t \), resulting in the matrices \( B_t^i \in \mathbb{R}^{d \times r_i} \) and \( A_t^i \in \mathbb{R}^{r_i \times k} \). The server stores these low-rank matrices, and their product gives the rank-\( r_i \) approximation of the global weight matrix: \( W_t^{r_i} = B_t^i A_t^i \). The truncation error is calculated as:
\begin{align}
\hat{r}_i(t) = \| W_t - W_t^{r_i} \|^2 ,
\end{align}
which is then used to adjust the aggregation weights in subsequent server-side updates.

Next, the client performs \( K \) local updates, obtaining the updated models \( B_{t, K}^i \) and \( A_{t, K}^i \), which are then uploaded to the server. Upon receiving the updates, the server reconstructs the local models \( W_{t+1}^i\) with \( B_{t, K}^i A_{t, K}^i\) using the uploaded matrices and aggregates them according to the following rule:
\begin{align}
W_{t+1} = \sum_{i=1}^N \mathbf{p}_i(t) \left[ W_t + \left( W_{t+1}^i - W_t^{r_i} \right) \right].
\end{align}

On the server side, the unbiased global model \( W_t \) serves as the aggregation baseline, addressing the model truncation bias introduced by client-specific rank approximations. To mitigate the gradient descent drift caused by these truncations, the aggregation weight \( \mathbf{p}_i(t) \) is computed using Eqn. (\ref{eq:weight}) based on the truncation error \( \hat{r}_i(t) \).  The training procedure is repeated iteratively for \( T \) communication rounds. Notably, the low-rank matrix multiplication used in the algorithm to reconstruct the model is computationally efficient. Moreover, the server aggregation process can be performed on-the-fly, incurring no additional storage overhead. The workflow of the FedHL algorithm is illustrated in Figure \ref{fig:fedlora}, and Table \ref{tab:comparison} provides a comparative summary of FedHL and other existing methods.

\section{Convergence Analysis}
\label{section:convergence}
In this section, we examine the convergence behavior of heterogeneous LoRA methods in FL, focusing on how truncation affects convergence guarantees. In contrast, we establish the convergence of the proposed FedHL method, demonstrating that its unbiased aggregation approach eliminates truncation bias, while its adjusted aggregation weights mitigate gradient descent drift. The following analysis is based on three standard assumptions:
\label{sec:convergence analysis}
\begin{assumption}
\label{assumption:1}
The loss function \( f \) is \( L \)-smooth with respect to \( W \). That is, for any \( W, \hat{W} \in \mathbb{R}^d \) and \( \xi \in \mathcal{D} \), the following holds:
\begin{align}
\left\|\nabla f(W;\xi) - \nabla f(\hat{W};\xi)\right\| \leq L\left\|W - \hat{W}\right\|.
\end{align}
\end{assumption}

\begin{assumption}\label{assumption:2} The local adapter gradient estimator \( \tilde{\nabla} f_i(W) \), computed using SGD with sample $\xi$, is unbiased and has a bounded variance \( \sigma_l^2 \) for all \( i \in [N] \). Specifically:
\begin{align}
    \mathbb{E}[\tilde{\nabla} f_i(W)] &= \nabla f_i(W),     \mathbb{E}\left\|\tilde{\nabla}  f_i(W)-{\nabla}f_i(W)\right\|^2\leq \sigma_l^2.
\end{align}
\end{assumption}

\begin{assumption}\label{assumption:3}  
The variance between local gradients $ \nabla f_i(W) $ and the global gradient $ \nabla f(W) $ is bounded by a constant. Specifically, for any weight matrix \( W \) and all clients \( i \in [N] \), we have:
\begin{align}
    \mathbb{E}\left[\left\|\nabla f(W) - \nabla f_i(W)\right\|^2\right] \leq \sigma_g^2,
\end{align}
where $ \sigma_g^2 > 0 $ quantifies data heterogeneity.  
\end{assumption}
Assumption \ref{assumption:1} indicates that the gradients of both the local loss function \( f_i \) and the global loss function \( f \) are Lipschitz continuous. The local adapter gradient is assumed to be unbiased and to have bounded variance, as stated in Assumption \ref{assumption:2}. Assumption \ref{assumption:3} states that the variance between the local and global gradients is bounded. In particular, the definition of \( \hat{r}_i \), which depends on \( t \), is simplified in the notation for the convergence analysis. Similar assumptions are widely adopted in LoRA training in FL \cite{wang2024flora,10763424,yi2023pfedlora,wang2025federated}. Note that throughout the paper, the expectation operator \( \mathbb{E} \) is taken with respect to all sources of randomness unless otherwise stated. We now present the main theorems for the convergence of existing heterogeneous LoRA methods and FedHL.

\begin{thm} Under Assumptions \ref{assumption:1}--\ref{assumption:3}, with the learning rate configured as \( \eta_t = \frac{1}{L K \sqrt{T}} \), the convergence behavior of the \textbf{existing heterogeneous LoRA} methodology in FL is formally characterized as follows:
\label{theorem:2}
\begin{equation}
\resizebox{\columnwidth}{!}{$
\begin{aligned}
\frac{1}{T} &\sum_{t=0}^{T-1} \mathbb{E}\|\nabla f(W_t)\|^2 \leq \frac{1}{C} \Bigg[ \frac{L(f(W_0)-f^*)}{\sqrt{T}} + \frac{L(L+1)}{2} \\
& \Bigg( \underbrace{2N\sqrt{T} \sum_{i \in [N]} (p_i)^2 (\hat{r}_i)^2}_{\text{Model Truncation Bias}}  + \underbrace{\frac{6ND_0^K}{\sqrt{T}} \sum_{i \in [N]} (p_i)^2 (\hat{r}_i)^2}_{\text{Accumulated Gradient Descent Drift}} \\
& \quad + \frac{12N \sigma_l^2}{L^2 \sqrt{T}} \sum_{i \in [N]} p_i^2 + \frac{48N \sigma_l^2 (D_0^K-1)}{L^2 K^2 T^\frac{3}{2} (D_0 - 1)} \sum_{i \in [N]} p_i^2 \Bigg) \\
& \quad + \frac{3N(\sigma_l^2 + 6K \sigma_g^2)}{2KT} \sum_{i \in [N]} p_i^2 + \frac{N \sigma_l^2}{2\sqrt{T}} \sum_{i \in [N]} p_i^2 \Bigg],
\end{aligned}
$}
\label{eq:example}
\end{equation}
where $C>0$ and \(D_0 \triangleq 4(1 + L^2 \eta_t^2) \geq 4\) are non-negative constants. The detailed proof can be found in Appendix \ref{Appendix:th2}. 
\end{thm}

\begin{remark} \normalfont It can be observed that the convergence bound in state-of-the-art heterogeneous LoRA algorithms for FL can largely be attributed to truncation effects, such as model truncation bias and accumulated gradient descent drift.
\end{remark}

\begin{remark} \normalfont More importantly, as \(T\) increases, the \textbf{\textit{Model Truncation Bias}} term, \( 2N\sqrt{T} \sum_{i \in [N]} (p_i)^2 (\hat{r}_i)^2 \), persists without a guaranteed reduction, leading to suboptimal convergence and hindered overall performance.
\end{remark}

\begin{table*}[ht]
    \centering
    \caption{Comparison of Heterogeneous LoRA Techniques.}
    \label{tab:comparison}
    \small
    \setlength{\tabcolsep}{4pt}
    \renewcommand{\arraystretch}{1.2}
    \begin{tabular}{
        @{} 
        m{2.5cm} 
        >{\centering\arraybackslash}m{1.85cm} 
        >{\centering\arraybackslash}m{1.85cm} 
        >{\centering\arraybackslash}m{1.85cm} 
        >{\centering\arraybackslash}m{1.85cm} 
        >{\centering\arraybackslash}m{6cm} 
        @{}
    }
        \toprule
        \textbf{Method} & \textbf{Aggregation Noise} & \textbf{Truncation Bias} & \textbf{Gradient Descent Drift} & \textbf{Strict Convergence Guarantee} & \textbf{Aggregation Method} \\
        \midrule
        \textbf{Zero-Padding}\cite{cho2023heterogeneous} & $\times$ & $\times$ & $\times$ & 
        $\times$ & $ W_{t+1} = \sum_{i=1}^N p_i[  W_t^{r_i} + (\hat{W}_{t+1}^i - W_t^{r_i})]$ \\
        
        \textbf{FlexLoRA} \cite{bai2024federated} & $\checkmark$ & $\times$ & $\times$ & 
        $\times$ & $ W_{t+1} = \sum_{i=1}^N p_i[W_t^{r_i} + (W_{t+1}^i - W_t^{r_i})] $ \\ 
        
        \textbf{FLoRA} \cite{wang2024flora} & $\checkmark$  & $\times$ & $\times$ & 
        $\times$ & $ W_{t+1} = \sum_{i=1}^N p_i[W_t^{r_i} + (W_{t+1}^i - W_t^{r_i})] $ \\
        
        \textbf{FedHL (Ours)} & $\checkmark$ & $\checkmark$ & $\checkmark$ & 
        $\checkmark$ & $ W_{t+1} = \sum_{i=1}^N \mathbf{p_i}(t)[W_t + (W_{t+1}^i - W_t^{r_i})] $ \\
        \bottomrule
    \end{tabular}
\end{table*}

\begin{thm}
Under Assumptions \ref{assumption:1}--\ref{assumption:3}, by choosing the learning rate \( \eta_t = \frac{1}{L K \sqrt{T}} \) as before, we aggregate heterogeneous LoRA using \textbf{FedHL}. The convergence behavior in FL can be characterized as follows: 
\label{theorem:1}
\begin{equation}
\resizebox{\columnwidth}{!}{$
\begin{aligned}
\frac{1}{T} &\sum_{t=0}^{T-1} \mathbb{E}\|\nabla f(W_t)\|^2 \leq \frac{1}{C} \Bigg[ \frac{L[f(W_0)-f^*]}{\sqrt{T}} + \frac{L(L+1)}{2} \\ 
& \Bigg( \underbrace{\frac{3ND_0^K}{\sqrt{T}} \sum_{i \in [N]} (p_i)^2(\hat{r}_i)^2}_{\text{Accumulated Gradient Descent Drift}} + \frac{24N \sigma_l^2 (D_0^K - 1)}{L^2 K^2 T^\frac{3}{2} (D_0 - 1)} \sum_{i \in [N]} p_i^2 \\
&\quad + \frac{6N \sigma_l^2}{L^2 \sqrt{T}} \sum_{i \in [N]} p_i^2 \Bigg) + \frac{3N(\sigma_l^2 + 6K\sigma_g^2)}{2KT} \sum_{i \in [N]} p_i^2 \\
&\quad + \frac{N \sigma_l^2}{2\sqrt{T}} \sum_{i \in [N]} p_i^2 \Bigg],
\end{aligned}
$}
\label{eq:th1}
\end{equation}
\noindent where \( C \) and \( D_0 \) are constants defined in Theorem \ref{theorem:2}, and the detailed proof can be found in Appendix \ref{app:th2}.
\end{thm}
\begin{remark}
\normalfont In the FedHL algorithm, the truncation bias term vanishes in comparison to existing state-of-the-art approaches. As $T$ increases, the convergence rate of the heterogeneous LoRA model is \(\mathcal{O}(1/\sqrt{T})\), which is consistent with the convergence rate of the homogeneous LoRA \cite{wu2024fedfmsl,yi2023fedlora}.
\end{remark} 

\begin{remark} \normalfont With Theorem \ref{theorem:1},  we observe that the \textit{\textbf{Accumulated Gradient Descent Drift}} terms, which contribute to the slowest convergence rate, affect the overall convergence behavior. This drift depends on the assigned weights \(p_i\) and the truncation error \(\hat{r}_i\). To achieve the tightest convergence bound, we minimize this drift by optimizing the weights \(\{p_i\}_{i=1}^N\).
\end{remark}
\begin{thm}
\label{theorem:3}
The optimal server aggregation weights that minimize the convergence rate upper bound in Eqn. (\ref{eq:th1}), accounting for the truncation errors \(\{\hat{r}_i\}_{i=1}^N\), with \(\epsilon > 0\) as a constant, are as follows:
\begin{align}
\label{eq:weight}
    p_i^* = \frac{\frac{1}{\hat{r}_i^2 + \epsilon}}{\sum_{j \in [N]} \frac{1}{\hat{r}_j^2 + \epsilon}}.
\end{align}
Note that \(\hat{r}_i\) corresponds to \(\hat{r}_i(t)\) in each communication round \(t\). The complete proof can be found in Appendix~\ref{Appendix:th3}.
\end{thm}

\begin{remark} \normalfont
Theorem \ref{theorem:3} provides a closed-form solution for the optimal aggregation weights in heterogeneous LoRA, derived by minimizing the upper bound of convergence in FedHL as presented in Theorem \ref{theorem:1}. Based on Theorem \ref{theorem:3}, we propose to dynamically adjust the aggregation weights of each round to alleviate gradient bias.
\end{remark}

\begin{remark} \normalfont
It can be observed that for any clients \(i\) and \(j\), if the truncation error satisfies \(\hat{r}_i \leq \hat{r}_j\), then the optimal aggregation weight satisfies \(p_i^* \geq p_j^*\). This assignment is consistent with our intuition that clients with smaller truncation errors should have higher aggregate weights.
\end{remark}

\begin{table}[htbp!]
\centering
\caption{Dataset Details for Training and Evaluation.}
\renewcommand{\arraystretch}{1.3}
\setlength{\tabcolsep}{4pt}
\begin{tabular}{@{}>{\centering\arraybackslash}p{2.6cm}
                >{\centering\arraybackslash}p{2.5cm}
                >{\centering\arraybackslash}p{0.8cm}
                >{\centering\arraybackslash}p{2.0cm}@{}}
\toprule
\textbf{Task} & \textbf{Training Dataset} & \textbf{Samples} & \textbf{Evaluation} \\
\midrule
Math Problem Solving & Fed-GSM8K~\cite{cobbe2021training} & 7473 & GSM8K-test~\cite{cobbe2021training} \\
Code Generation & Fed-CodeAlpaca~\cite{chaudharycode} & 20000 & HumanEval~\cite{chen2021evaluating} \\
Generic Language QA & Fed-Dolly~\cite{databricks2023dolly} & 15015 & Rouge-L~\cite{qin2024federated} \\
\bottomrule
\end{tabular}
\label{tab:dataset_details}
\end{table}

\section{Experiments}
\label{section:experimental results}
\subsection{Experiments setup}
\noindent \textbf{Model and Baselines.}
In this study, we leverage the pre-trained Llama-2 7B model \cite{touvron2023llama} for downstream tasks \footnote{%
While the newest models like the recent LLAMA and Qwen series exhibit outstanding performance owing to pre-training on extensive datasets, fine-tuning them poses risks of overfitting and compromised generalization due to the lack of high-quality datasets from public sources.}. We choose Zero-Padding \cite{cho2023heterogeneous}, FlexLoRA \cite{bai2024federated}, and FLoRA \cite{wang2024flora} as baseline methods. Specifically, Zero-Padding tackles the challenge of heterogeneous matrix aggregation by padding zeros in LoRA modules, FlexLoRA employs a multiply-then-aggregate approach to aggregate heterogeneous LoRA modules, and FLoRA stacks all clients' LoRA modules to meet the requirements of matrix multiplication. Furthermore, our proposed FedHL adopts a full-rank model to perform unbiased aggregation to eliminate truncation bias and optimizes the aggregation weights to minimize the accumulated training drift. We also adopt FedIT \cite{zhang2024towards} as a baseline for homogeneous LoRA settings.

\begin{table*}[ht]
\centering
\caption{Performance comparison of different algorithms under various heterogeneous LoRA configurations (10 clients), with evaluation scores (\%) ± standard deviation (\%).}
\label{tab:experimental_results}
\begin{adjustbox}{width=\textwidth}
\renewcommand{\arraystretch}{0.2}
\begin{tabular}{l c *{4}{c}}
\toprule
\rowcolor{gray!15}
\textbf{Algorithm} & \textbf{\makecell{LoRA Rank \\ Heterogeneity Level}} & \textbf{\makecell{Fed-GSM8K}} & \textbf{\makecell{Fed-CodeAlpaca}} & \textbf{\makecell{Fed-Dolly}} & \textbf{\makecell{Average\\Performance}} \\
\midrule
\multirow{2}{*}[-0.5ex]{\textbf{Zero-Padding}} 
& \cellcolor{gray!20}Moderate  & $32.84 ^{\pm 0.42}$ & $19.26 ^{\pm 0.27}$ & $32.08 ^{\pm 0.25}$ & $28.06^{\pm 0.32}$ \\
& \cellcolor{gray!35}High      & $27.36 ^{\pm 0.54}$ & $18.67 ^{\pm 0.32}$ & $31.39 ^{\pm 0.13}$ & $25.81^{\pm 0.37}$ \\
\midrule

\multirow{2}{*}[-0.5ex]{\textbf{FlexLoRA}} 
& \cellcolor{gray!20}Moderate  & $33.58^{\pm 0.38}$ & $19.64^{\pm 0.25}$ & $32.36^{\pm 0.21}$ & $28.53^{\pm 0.29}$ \\
& \cellcolor{gray!35}High      & $30.11^{\pm 0.56}$ & $19.55^{\pm 0.27}$ & $32.22^{\pm 0.15}$ & $26.96^{\pm 0.37}$ \\
\midrule

\multirow{2}{*}[-0.5ex]{\textbf{FLoRA}} 
& \cellcolor{gray!20}Moderate  & $32.97^{\pm 0.21}$ & $19.66^{\pm 0.31}$ & $32.28^{\pm 0.23}$ & $28.30^{\pm 0.25}$ \\
& \cellcolor{gray!35}High      & $29.62^{\pm 0.60}$ & $17.74^{\pm 0.11}$ & $32.04^{\pm 0.14}$ & $26.80^{\pm 0.36}$ \\
\midrule

\multirow{2}{*}[-0.5ex]{\textbf{FedHL}} 
& \cellcolor{gray!20}Moderate  & $\textbf{34.34}^{\pm 0.33}$ & $\textbf{20.97}^{\pm 0.17}$ & $\textbf{32.65}^{\pm 0.19}$ & $\textbf{29.32}^{\pm 0.24}$ \\
& \cellcolor{gray!35}High      & $\textbf{32.29}^{\pm 0.45}$ & $\textbf{20.29}^{\pm 0.10}$ & $\textbf{32.63}^{\pm 0.12}$ & $\textbf{28.40}^{\pm 0.28}$ \\
\midrule

\end{tabular}
\end{adjustbox}
\end{table*}
\noindent \textbf{Datasets.} In this study, based on the federated LLM dataset benchmark proposed by FederatedScope-LLM \cite{kuang2024federatedscope}, we utilize the following datasets under this benchmark as our evaluation datasets: Fed-GSM8K, Fed-CodeAlpaca, and Fed-Dolly datasets.
\begin{itemize} 

\item  \textbf{Fed-GSM8K.} The Fed-GSM8K dataset \cite{cobbe2021training} consists of grade-school math problems that require multi-step reasoning. We evaluate models on the GSM8K test dataset, measuring their accuracy in solving math problems.

\item \textbf{Fed-CodeAlpaca.} The Fed-CodeAlpaca \cite{chaudharycode} dataset consists of coding exercises across nine programming languages: C, C\#, C++, Go, Java, PHP, Pascal, Python, and Scala. The dataset is divided into nine subsets, each corresponding to a specific language and labeled accordingly. We employ HumanEval \cite{chen2021evaluating} as the evaluation standard.

\item \textbf{Fed-Dolly.}  Fed-Dolly leverages the Databricks Dolly 15k dataset \cite{databricks2023dolly}, which comprises diverse categories of instruction-response tasks, such as open-ended question answering, classification, and general QA. Similar to \cite{qin2024federated, shu2024ferret}, we utilize the Rouge-L metric for evaluation. 
\end{itemize}

\noindent \textbf{Data Partition and Rank Heterogeneity Configuration.}
\begin{itemize} 

\item \textbf{Case 1: Cross-Silo FL with 10 Clients.} For the datasets Fed-GSM8K and Fed-CodeAlpaca, we apply a IID partitioning where each client is assigned equally sized data blocks randomly. For Fed-Dolly, we employ Dirichlet-based data partitioning with $\alpha = 0.3$, a widely used method to simulate non-IID data distributions in real-world FL scenarios~\cite{ling2024convergence}. To simulate heterogeneous LoRA rank distributions, we follow the framework proposed in \cite{wang2024flora}, considering two heterogeneity levels: moderate and high. Consistent with prior studies \cite{wang2024flora}, the moderate heterogeneity scenario adopts the following rank distribution: [32, 24, 20, 16, 16, 12, 12, 12, 8, 8]. For the high heterogeneity scenario, the rank distribution is as follows: [64, 32, 16, 16, 8, 8, 4, 4, 4, 4].

\item \textbf{Case 2: Cross-Device FL with 100 Clients.} Since the Fed-GSM8K dataset does not contain labels, we apply Dirichlet distribution with $\alpha = 0.3$ to the labeled Fed-CodeAlpaca and Fed-Dolly datasets in order to simulate the non-IID data distribution. We implement partial client participation with a sampling rate of 15\% and 30\%. The rank distribution is proportionally scaled from the original FLoRA setting \cite{wang2024flora} to 100 clients while maintaining an average rank of 16. In our experiments, clients with larger LoRA ranks are assigned more data and a more skewed local distribution, simulating the non-IID condition where these clients hold a more diverse sample.

\end{itemize}
\noindent \textbf{(Hyper)Parameters.} In the Cross-Silo FL scenario, \(T\) is set to 20 rounds for both moderate and high heterogeneity in Fed-GSM8K, 10 rounds for Fed-CodeAlpaca, and 5 rounds for Fed-Dolly. In the Cross-device FL scenario, \(T\) is set to 20 rounds for Fed-CodeAlpaca and 15 rounds for Fed-Dolly in both heterogeneity settings.
The number of local epochs \(K\) is set to 3 for Fed-GSM8K, 1 for Fed-CodeAlpaca, and 2 for Fed-Dolly. The number of training rounds is configured according to the complexity of the task, while avoiding overfitting and reducing generalization when fine-tuning large models. The learning rates \(\eta_t\) are set at \(5 \times 10^{-4}\) for Fed-GSM8K, and \(2 \times 10^{-4}\) for the other datasets, with a batch size of 1 for all. The weight decay is set at 0.01 for Fed-GSM8K and Fed-CodeAlpaca, and 0.001 for Fed-Dolly. The aggregated weights are smoothed using a softmax function with a temperature of 1 to moderate extreme contributions. All hyperparameters are optimized via an extensive hyperparameter search, including baselines. The LoRA configuration targets the ``q\_proj'' and ``v\_proj'' modules across all datasets. All experimental results were obtained using 8 $\times$ A800 80GB GPUs.

\subsection{Performance Evaluation in Cross-Silo FL}

\noindent \textbf{Comparison with Baseline Methods.} In Table~\ref{tab:experimental_results}, we compare FedHL with three baseline methods on Fed-GSM8K, Fed-CodeAlpaca, and Fed-Dolly under varying levels of LoRA heterogeneity. The results demonstrate that FedHL achieves state-of-the-art performance across all datasets and heterogeneity levels, with overall average scores exceeding those of the baselines by approximately 1--3\%. In addition, it can be observed that in highly heterogeneous settings, Zero-Padding suffers the most noticeable drop, indicating that simple zero-padding cannot effectively handle high heterogeneity. FlexLoRA and FLoRA, which do not explicitly address truncation and gradient biases under heterogeneous LoRA conditions, also underperform compared to FedHL. Figures \ref{fig:loss_gsm8k_10client} and \ref{fig:loss_dolly_10_clients} illustrate the training loss value curves of all algorithms during training, showing that FedHL converges faster than other baselines.
\setcounter{figure}{2}
\begin{figure}[!]
    \centering
    \begin{minipage}[b]{0.48\columnwidth}
        \centering
        \includegraphics[width=\textwidth]{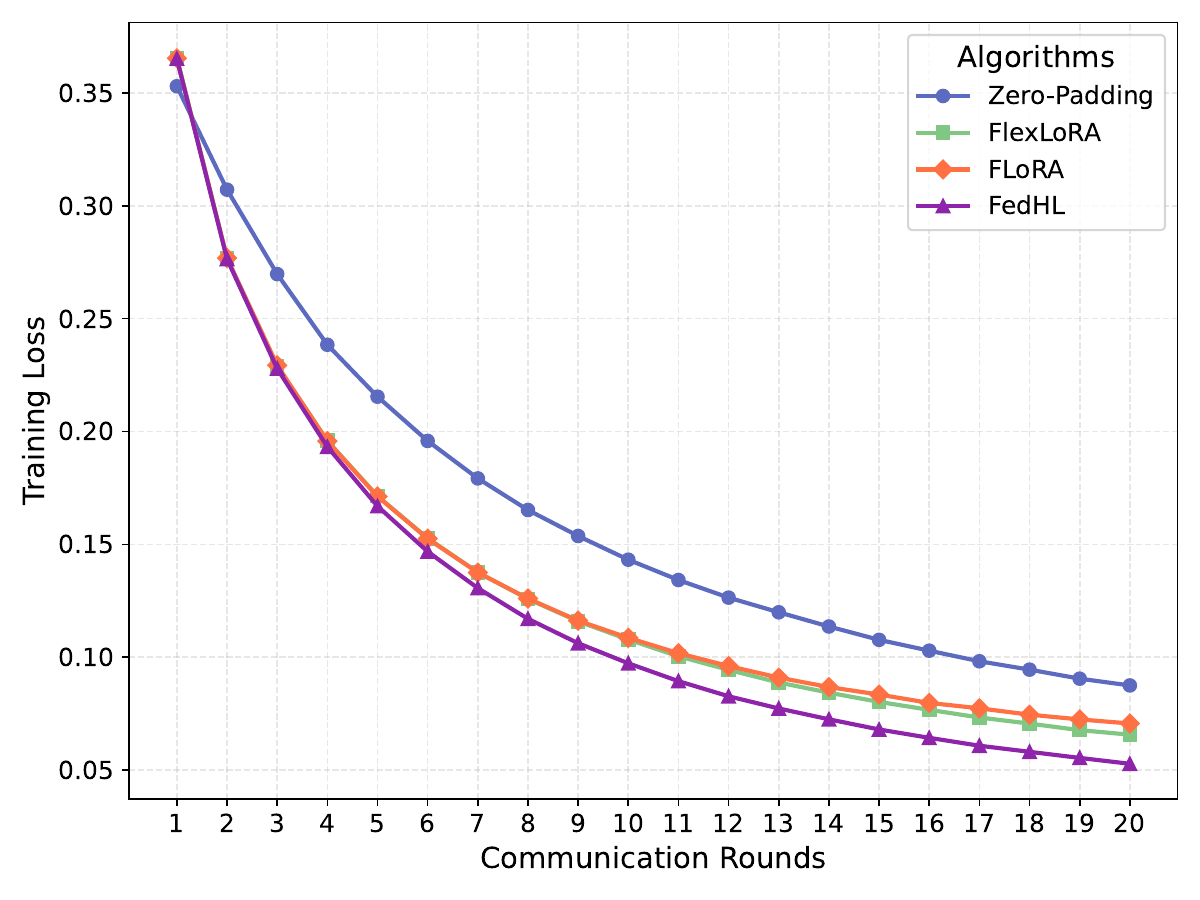}
    \end{minipage}
    \hfill
    \begin{minipage}[b]{0.48\columnwidth}
        \centering
        \includegraphics[width=\textwidth]{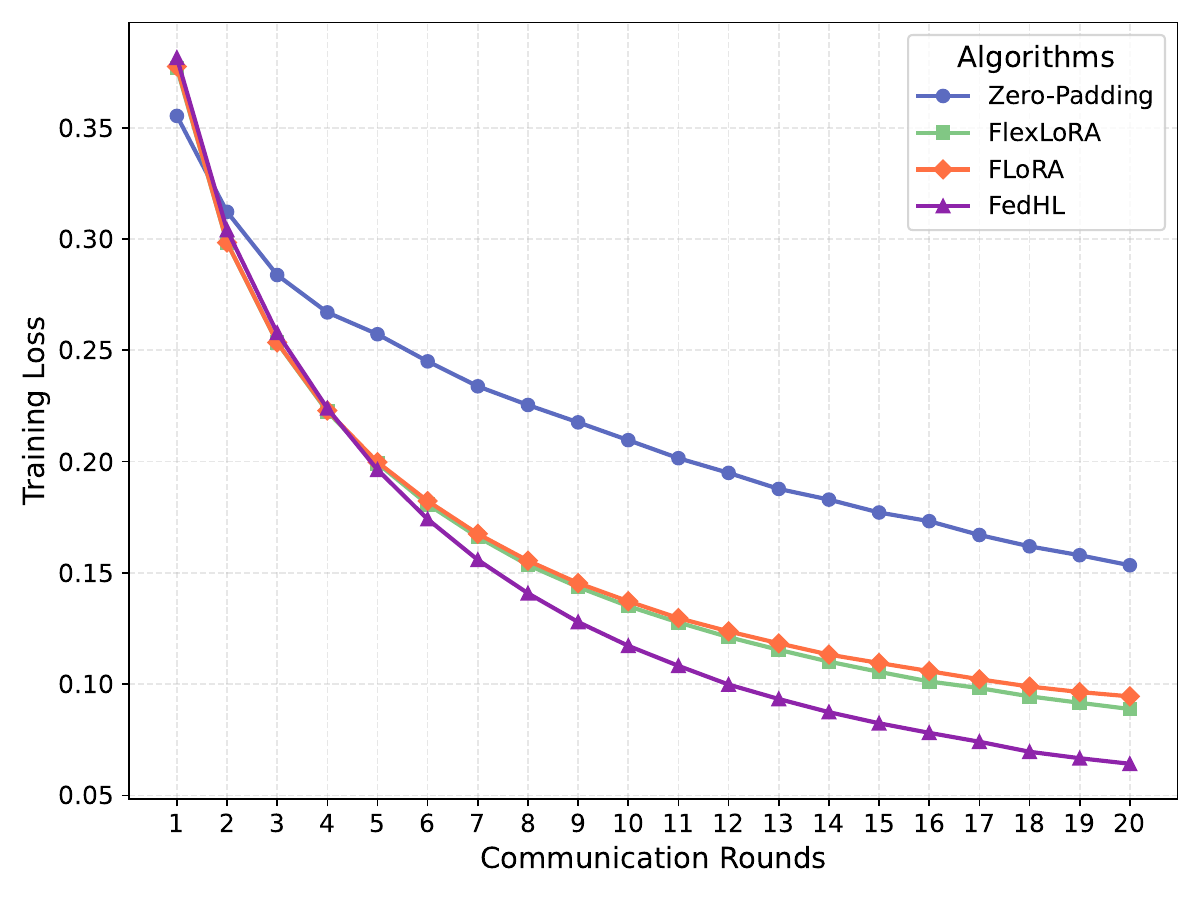}
    \end{minipage}
    \caption{Training Loss for LoRA-Based Configurations under Moderate and High Heterogeneity (Fed-GSM8K).}
    \label{fig:loss_gsm8k_10client}
\end{figure}

\setcounter{figure}{4}
\begin{figure}[!]
    \centering
    \begin{minipage}[b]{0.5\columnwidth}
        \centering
        \includegraphics[width=\textwidth]{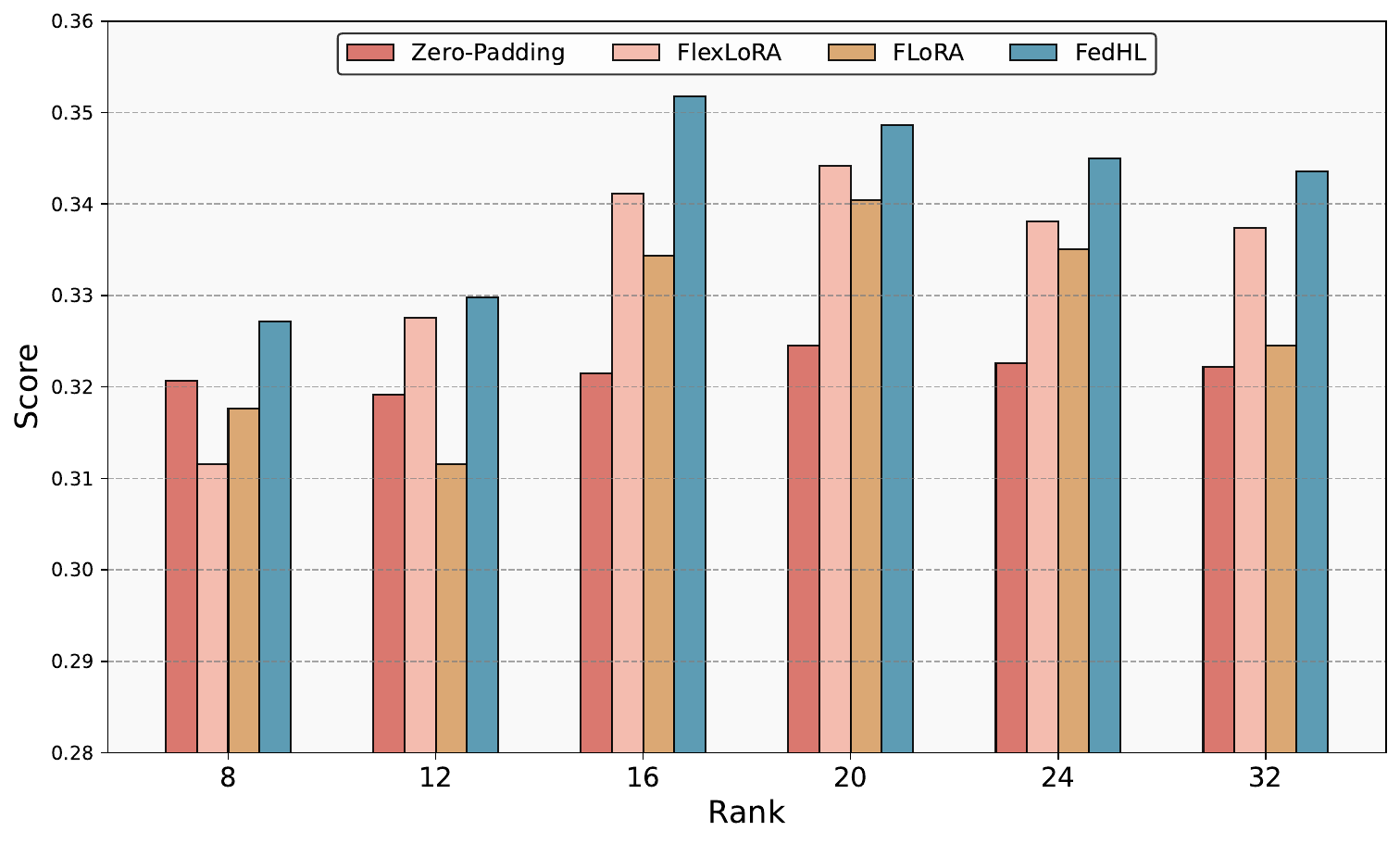}
    \end{minipage}%
    \hfill
    \begin{minipage}[b]{0.5\columnwidth}
        \centering
        \includegraphics[width=\textwidth]{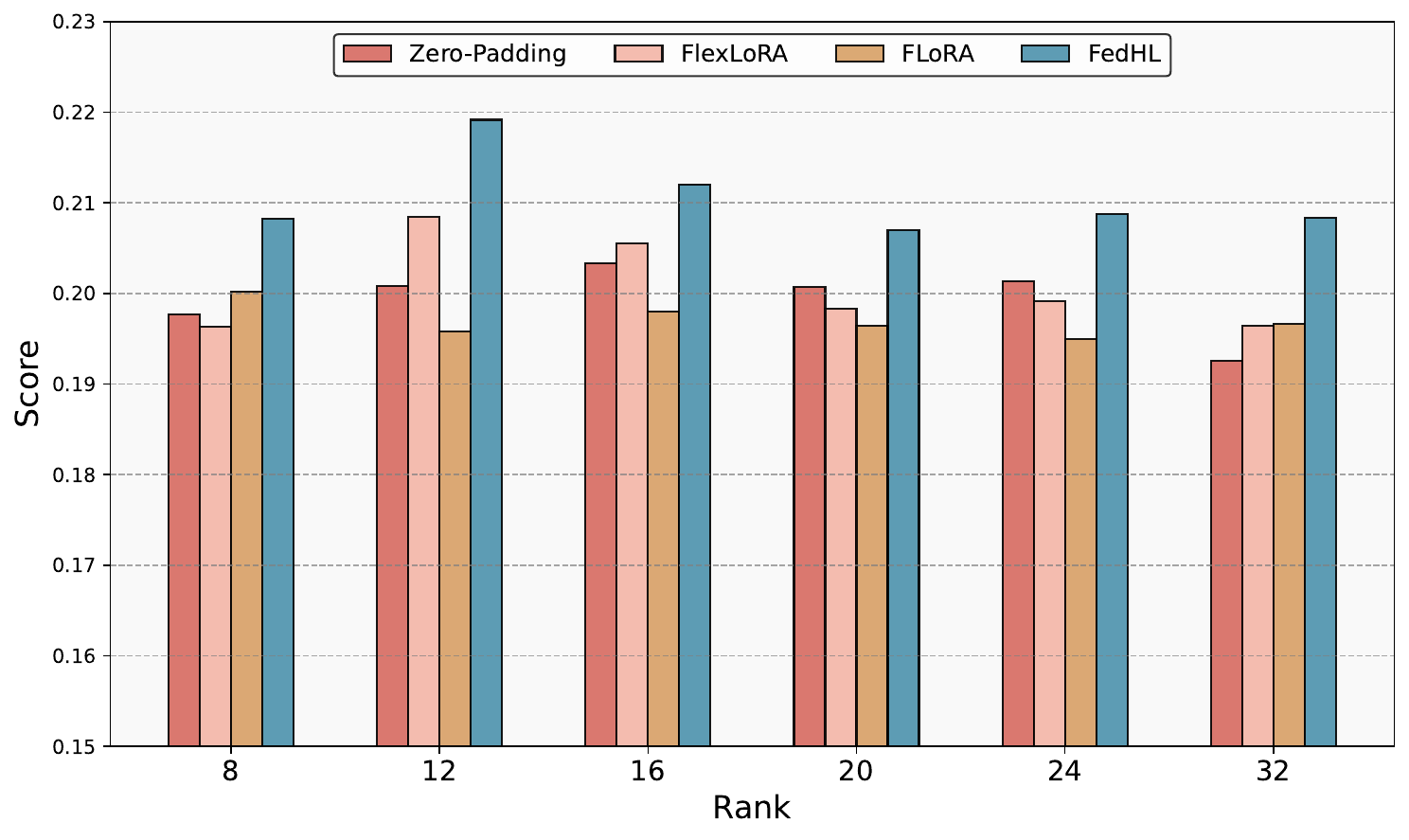}
    \end{minipage}
    \caption{Comparison of Client Model Scores Across Different LoRA Ranks in a Cross-Device Setting. Left: Fed-GSM8K. Right: Fed-CodeAlpaca.}
    \label{fig:per_client_10_client}
\end{figure}

\noindent \textbf{Per-Client Performance Analysis.} Figure~\ref{fig:per_client_10_client} illustrates the performance of the final client models under moderate LoRA rank configurations. Consistently, FedHL demonstrates superior and robust performance compared to other algorithms under different rank configurations on the final client models. Specifically, the left side of Figure~\ref{fig:per_client_10_client} presents the scores of client models trained on the Fed-GSM8K dataset, where the client with rank 16 achieves the highest performance. Similarly, the right side of Figure~\ref{fig:per_client_10_client} depicts the performance of client models on the Fed-CodeAlpaca dataset, with the client at rank 12 achieving the best score. The experimental results suggest that excessively small ranks lead to underfitting, as the model lacks sufficient capacity to learn complex patterns. In contrast, excessively large ranks can induce overfitting, degrading both generalization and learning performance.

\noindent \textbf{Impact of LoRA Ranks.} We conduct experiments with various LoRA rank configurations, including Homogeneous (16 ranks per client), High, Moderate, and Extremely High (64 ranks for the first two clients and 4 for the rest). Using the \textit{Moderate} rank allocation as the baseline, we systematically evaluated the performance of these configurations under a Dirichlet distribution ($\alpha = 0.3$) in the Fed-Dolly dataset. As observed in Table~\ref{tab:moderate_baseline}, the findings indicate that different configurations yield varying performance levels, and specifically, the Moderate configuration consistently outperforms the other configurations. Furthermore, the experimental results also show that FedHL exhibits the best and most stable performance across all ranking configurations, with a maximum score deviation of only 0.07. Future work could investigate advanced optimization strategies for rank allocation to further enhance model performance.

\setcounter{figure}{3}
\begin{figure}[!]
    \centering
    \begin{minipage}[b]{0.48\columnwidth}
        \centering
        \includegraphics[width=\textwidth]{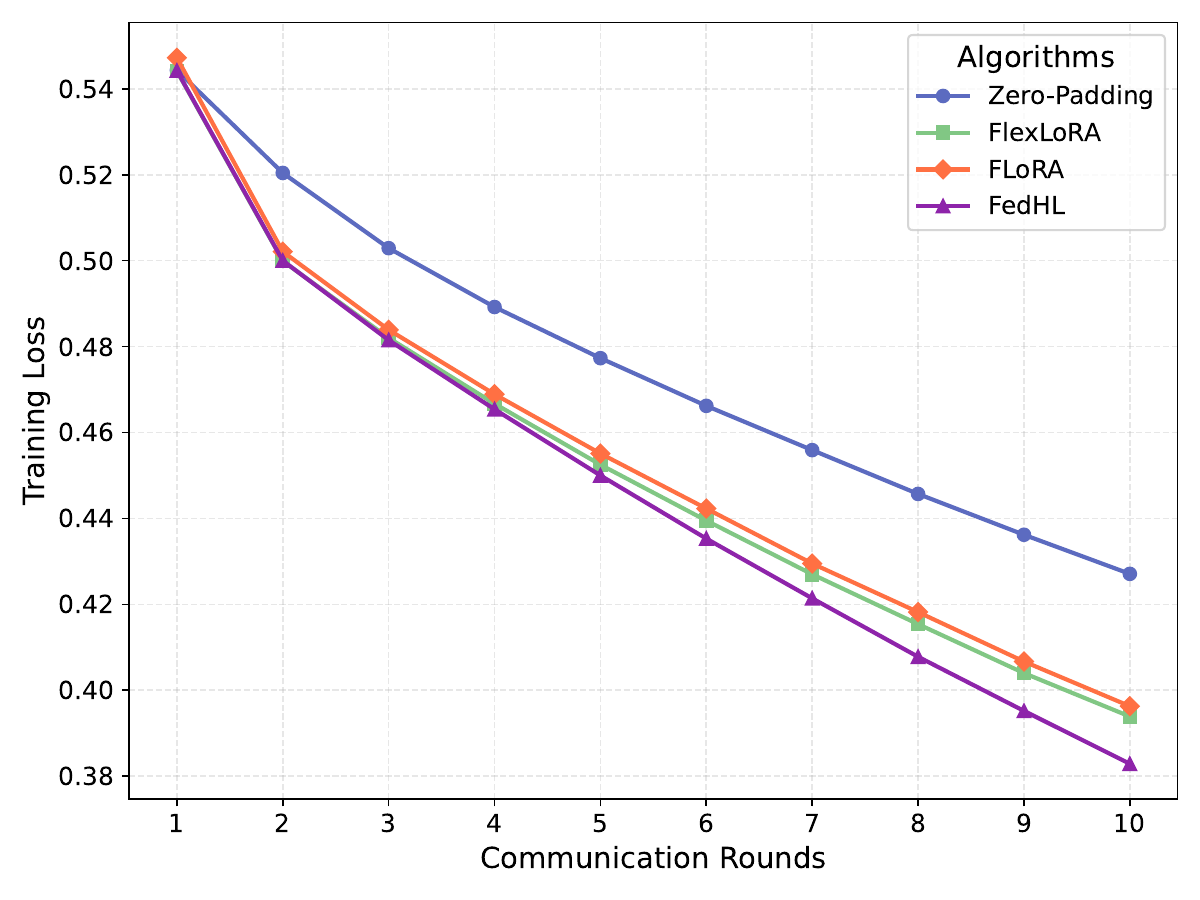}
    \end{minipage}
    \hfill
    \begin{minipage}[b]{0.48\columnwidth}
        \centering
        \includegraphics[width=\textwidth]{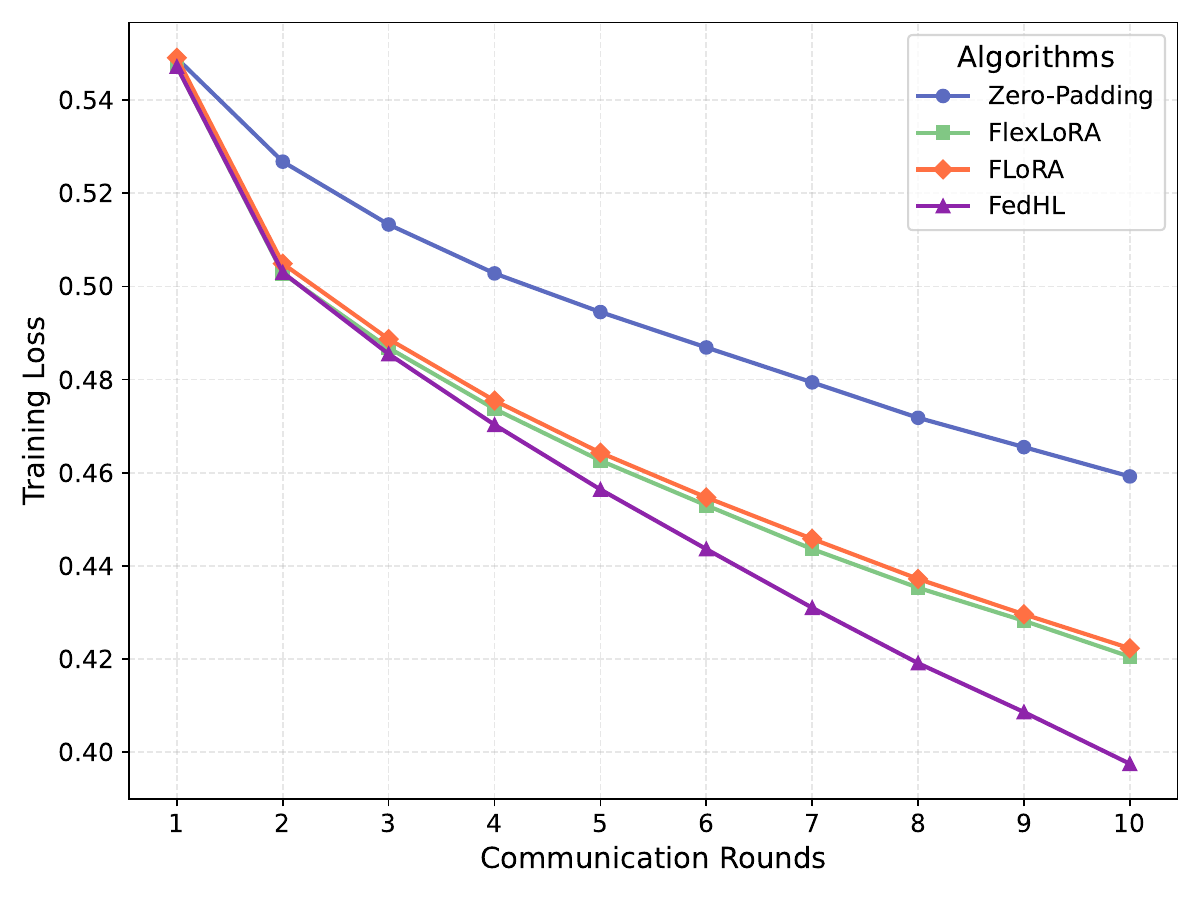}
    \end{minipage}
    \caption{Training Loss for LoRA-Based Configurations under Moderate and High Heterogeneity (Fed-CodeAlpaca).}
    \label{fig:loss_dolly_10_clients}
\end{figure}

\setcounter{figure}{5}
\begin{figure}[!]
    \centering
    \begin{minipage}[b]{0.5\columnwidth}
        \centering
        \includegraphics[width=\textwidth]{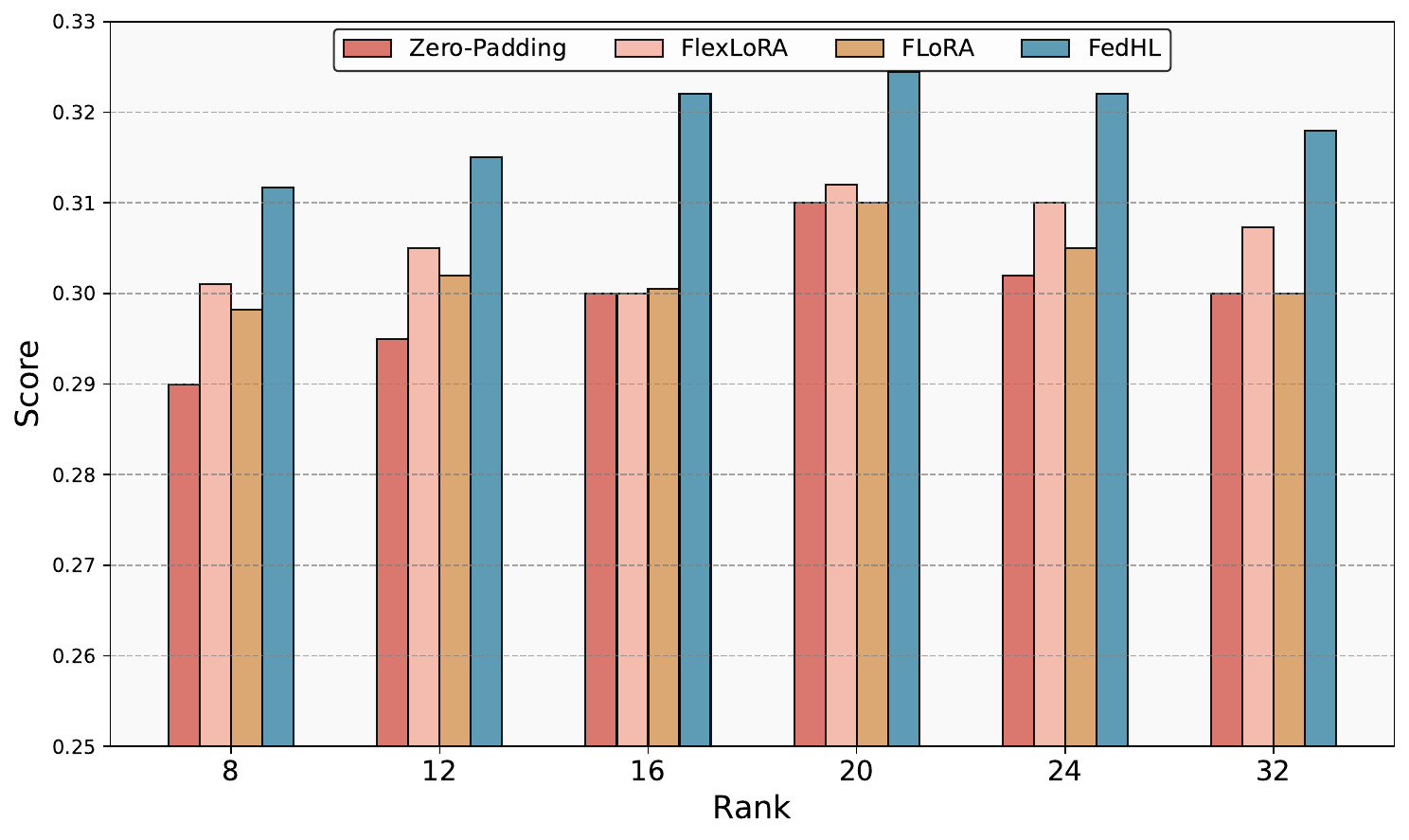}
    \end{minipage}%
    \hfill
    \begin{minipage}[b]{0.5\columnwidth}
        \centering
        \includegraphics[width=\textwidth]{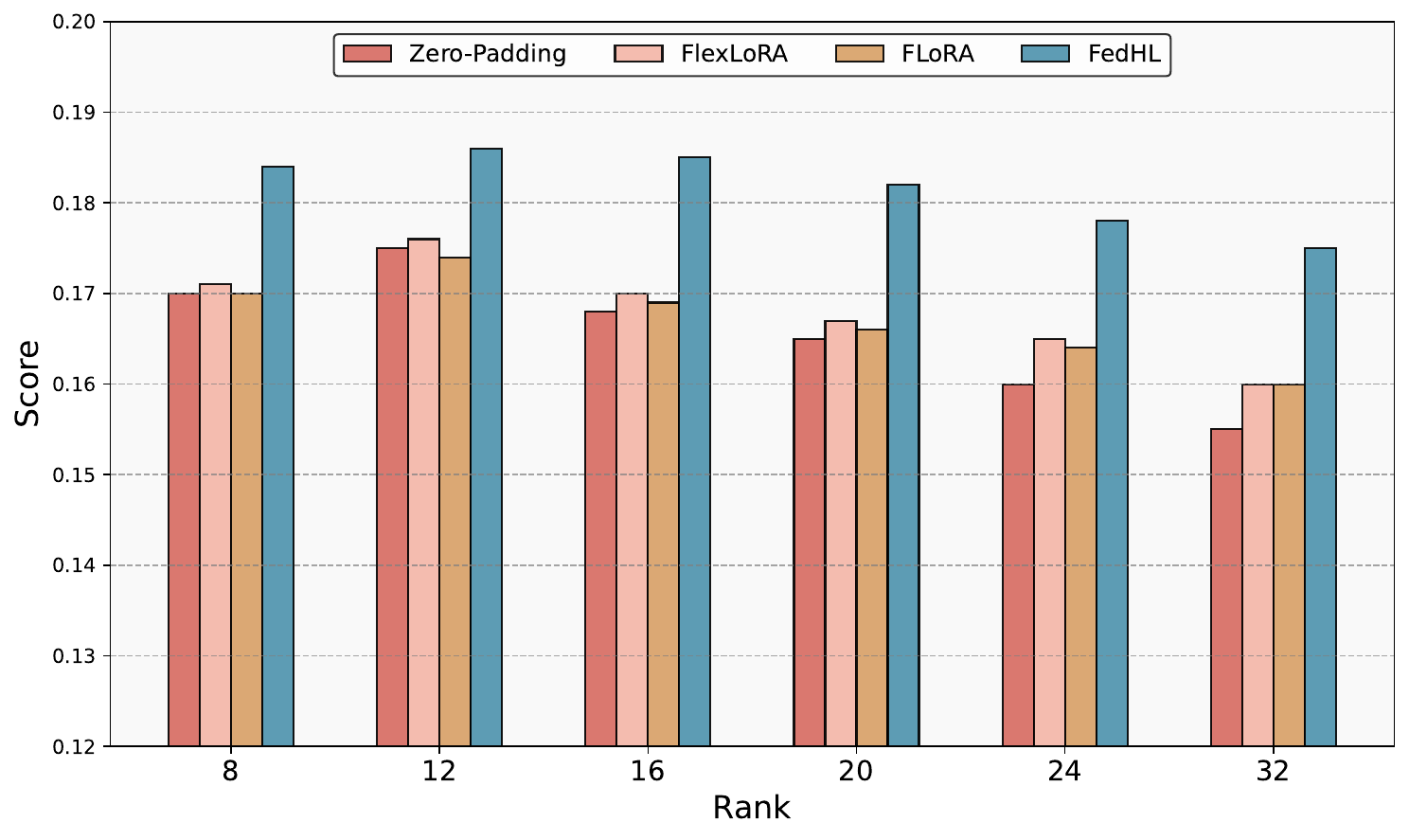}
    \end{minipage}
    \caption{Comparison of Client Model Scores Across Different LoRA Ranks in a Cross-Silo Setting. Left: Fed-Dolly. Right: Fed-CodeAlpaca.}
    \label{fig:per_client_100_client}
\end{figure}
\begin{table}[!]
\centering
\large
\renewcommand{\arraystretch}{1.3} 
\setlength{\tabcolsep}{4pt}     
\caption{Performance Comparison Across Four Rank Configurations (Homogeneous, Moderate, High, Extremely High). Using \underline{\textbf{Moderate}} as the Baseline.}
\label{tab:moderate_baseline}
\resizebox{\linewidth}{!}{ 
\begin{tabular}{lcccc}
\toprule 
\rowcolor{gray!15}
\textbf{Method} & \textbf{Homogeneous} & \textbf{High} & \textbf{Extremely High} & \textbf{Moderate} \\
\midrule
\textbf{Zero-Padding} & 32.04 \,(-0.04) & 31.39 \,(-0.69) & 30.63 \,(-1.45) & \textbf{32.08} \\

\textbf{FlexLoRA}     & 32.33 \,(-0.03) & 32.22 \,(-0.14) & 32.15 \,(-0.21) & \textbf{32.36} \\

\textbf{FLoRA}        & 32.18 \,(-0.10) & 32.04 \,(-0.24) & 32.08 \,(-0.20) & \textbf{32.28} \\

\textbf{FedHL}        & 32.58 \,(-0.07) & 32.63 \,(-0.02) & \textbf{32.65 (0.00)} & \textbf{32.65} \\
\bottomrule
\end{tabular}}
\end{table}

\begin{table}[!]
\centering
\caption{Performance Comparison under Cross-Device FL with 100 Clients at 15\% Client Participation Rate, with evaluation scores (\%) ± standard deviation (\%).}
\label{tab:cross_device_results}
\resizebox{\columnwidth}{!}{ 
\begin{tabular}{l c cc}
\toprule
\rowcolor{gray!15}
\textbf{Algorithm} & \textbf{\makecell{LoRA Rank \\
Heterogeneity Level}} & \textbf{\makecell{Fed-CodeAlpaca}} & \textbf{\makecell{Fed-Dolly}} \\
\midrule

\multirow{2}{*}[-0.3ex]{\textbf{Zero-Padding}} 
& \cellcolor{gray!20}Moderate   & 16.73$^{\pm 0.21}$ & $29.99^{\pm 0.13}$ \\
& \cellcolor{gray!35}High   & 16.24$^{\pm 0.24}$ & $30.27^{\pm 0.27}$ \\
\midrule

\multirow{2}{*}[-0.3ex]{\textbf{FlexLoRA}} 
& \cellcolor{gray!20}Moderate    & 17.21$^{\pm 0.20}$ & $31.09^{\pm 0.14}$ \\
& \cellcolor{gray!35}High   & 16.84$^{\pm 0.21}$ & $30.44^{\pm 0.23}$ \\
\midrule

\multirow{2}{*}[-0.3ex]{\textbf{FLoRA}} 
& \cellcolor{gray!20}Moderate   & 17.03$^{\pm 0.24}$ & $30.82^{\pm 0.15}$ \\
& \cellcolor{gray!35}High   & 16.57$^{\pm 0.24}$ & $30.23^{\pm 0.22}$ \\
\midrule

\multirow{2}{*}[-0.3ex]{\textbf{FedHL}} 
& \cellcolor{gray!20}Moderate   & $\textbf{18.93}^{\pm 0.14}$ & $\textbf{32.47}^{\pm 0.12}$ \\
& \cellcolor{gray!35}High   & $\textbf{18.65}^{\pm 0.18}$ & $\textbf{32.32}^{\pm 0.19}$ \\
\bottomrule
\end{tabular}
}
\end{table}

\subsection{Performance Evaluation in Cross-Device FL}
\noindent \textbf{Comparison with Baseline Methods.}  In the cross-device setting, we follow the heterogeneous rank allocation strategies introduced in FLoRA and extend them to 100 clients \cite{bai2024federated}, simulating real-world device capabilities by adopting both moderate and high rank distributions across clients. All experiments are conducted with 100 clients under partial participation, where 15\% of clients are sampled per round. The training loss versus rounds is shown in Figs. \ref{fig:loss_dolly_100_client} and \ref{fig:loss_code_alpaca_100_clients}.

Table~\ref{tab:cross_device_results} summarizes the performance across two datasets: Fed-CodeAlpaca and Fed-Dolly. Our proposed algorithm, FedHL, consistently outperforms all baselines under both the moderate and high rank configurations. Specifically, on the Fed-Dolly dataset, FedHL significantly outperforms all baselines, scoring 32.47 in the moderate-rank setting and 32.32 in the high-rank setting. Notably, under highly rank-heterogeneous settings, several baselines experience performance degradation, while FedHL remains robust. Compared to the second-best method, FlexLoRA, FedHL improves the score by 1.88 in the high-rank setting and surpasses FLoRA by 2.09. On the Fed-CodeAlpaca dataset, FedHL achieves a score of 18.93 under moderate configurations, outperforming FlexLoRA (17.21) and FLoRA (17.03) by 9.9\% and 11.2\%, respectively. Even under heavy high-rank settings, FedHL maintains a robust score of 18.65, demonstrating a 10.7\% relative improvement over the second-best method.
\setcounter{figure}{6}
\begin{figure}[!]
    \centering
    \begin{minipage}[b]{0.48\columnwidth}
        \centering
        \includegraphics[width=\textwidth]{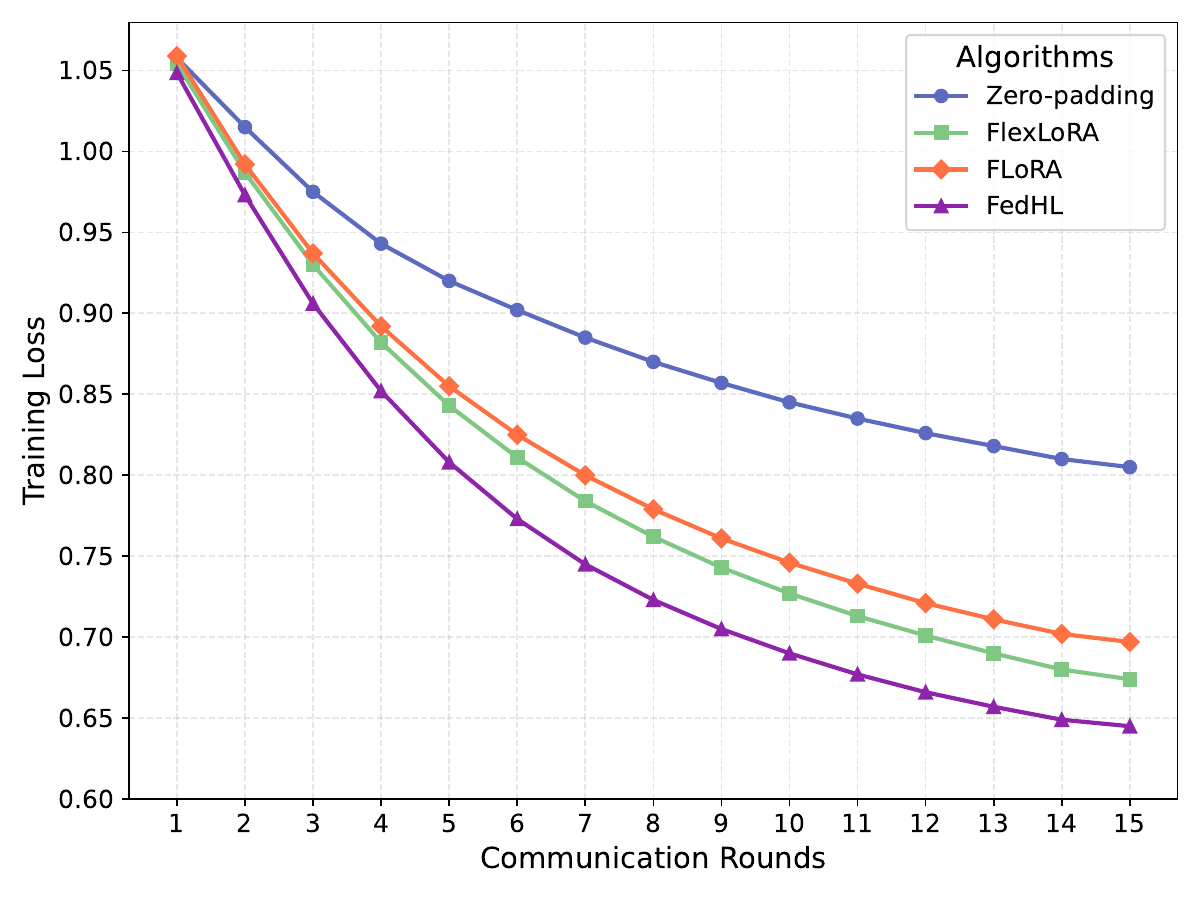}
    \end{minipage}
    \hfill
    \begin{minipage}[b]{0.48\columnwidth}
        \centering
        \includegraphics[width=\textwidth]{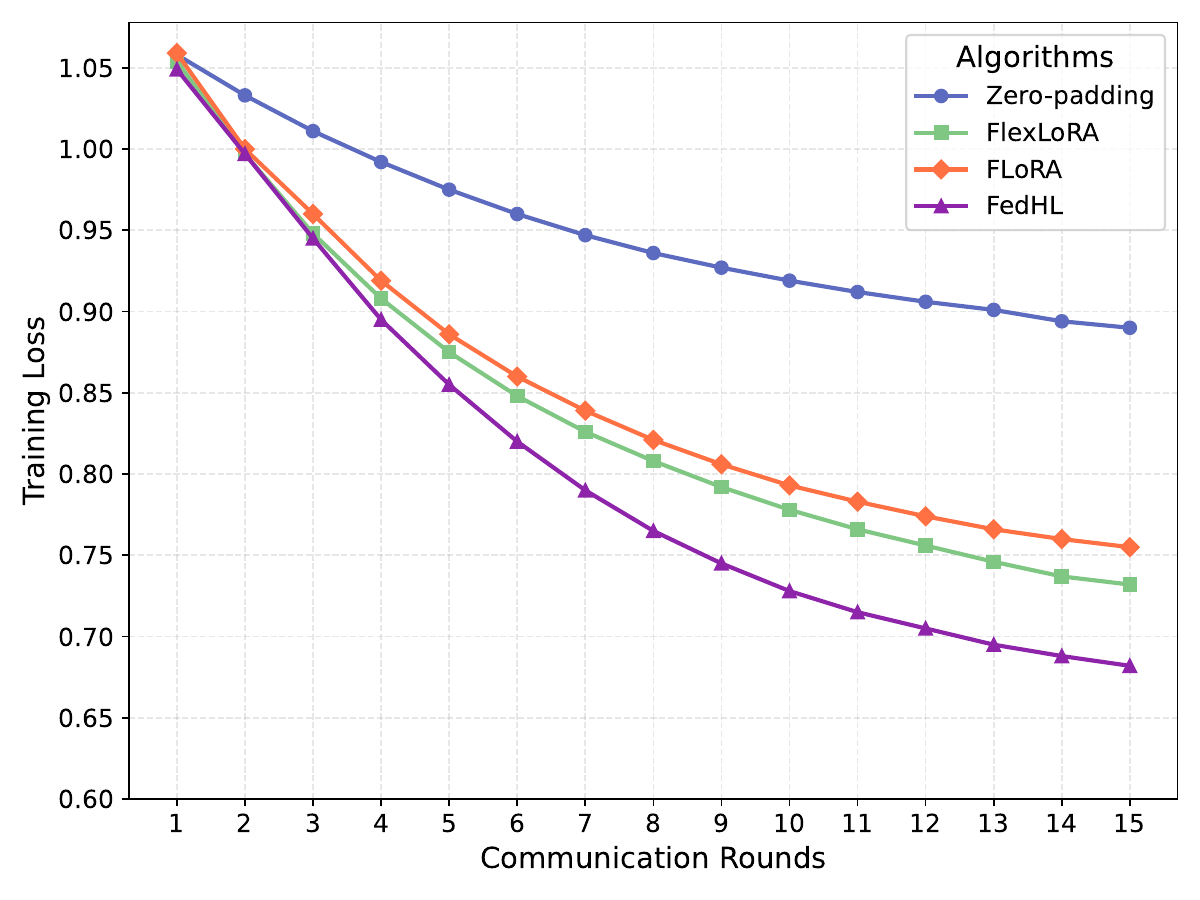}
    \end{minipage}
    \caption{Training Loss for LoRA-Based Configurations under Moderate and High Heterogeneity (Fed-Dolly).}
    \label{fig:loss_dolly_100_client}
\end{figure}
\setcounter{figure}{8}
\begin{figure}[!]
    \centering
    \begin{minipage}[b]{0.5\columnwidth}
        \centering
        \includegraphics[width=\textwidth]{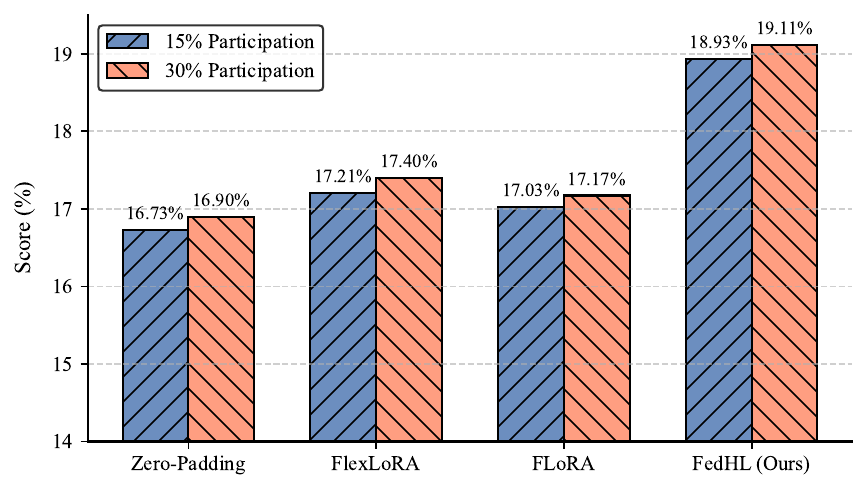}
    \end{minipage}%
    \hfill
    \begin{minipage}[b]{0.5\columnwidth}
        \centering
        \includegraphics[width=\textwidth]{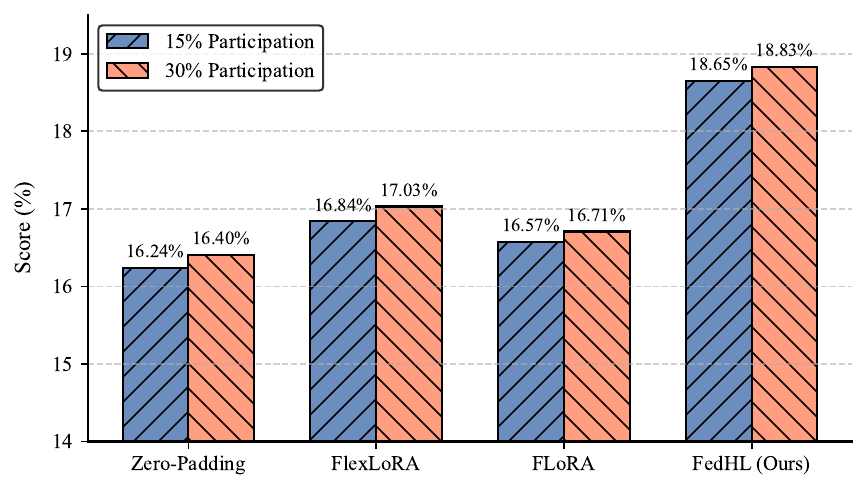}
    \end{minipage}
    \caption{Performance comparison under client participation rates (15\% vs. 30\%) on the Fed-CodeAlpaca dataset for  different rank levels: Left for moderate, Right for high.}
    \label{fig:compare_code}
\end{figure}

\noindent \textbf{Client-Specific Performance Evaluation.} Figure~\ref{fig:per_client_100_client} illustrates the performance of the final client models across different moderate LoRA rank configurations. In all cases, FedHL consistently achieves superior and stable results compared to other algorithms. Specifically, the left panel of Figure~\ref{fig:per_client_100_client} shows the performance of client models trained on the Fed-Dolly dataset, where the client with a rank of 20 achieves the highest performance. On the right side of the figure, the performance of client models on the Fed-CodeAlpaca dataset is presented, with the client at rank 12 attaining the best score, consistent with the findings in the cross-silo FL scenarios. The findings suggest that small ranks may cause underfitting by limiting model capacity, whereas large ranks can lead to overfitting.

\noindent \textbf{Analysis of Client Participation Impact.} We perform experiments under cross-device FL scenarios to investigate the impact of client participation rates. For each dataset, we configure participation rates at 15\% and 30\% to simulate different device availability levels. As shown in Figs.~\ref{fig:compare_code} and~\ref{fig:compare_dolly}, increasing the participation rate from 15\% to 30\% consistently improves performance across all baselines on both datasets. Specifically, as participation increases, the average score of all baselines improves by $\Delta_{+0.17\%}$ on Fed-Dolly and $\Delta_{+0.16\%}$ on Fed-CodeAlpaca. Notably, our proposed algorithm consistently outperformed all baselines while maintaining stable performance across varying participation rates.

\setcounter{figure}{7}
\begin{figure}[!]
    \centering
    \begin{minipage}[b]{0.48\columnwidth}
        \centering
        \includegraphics[width=\textwidth]{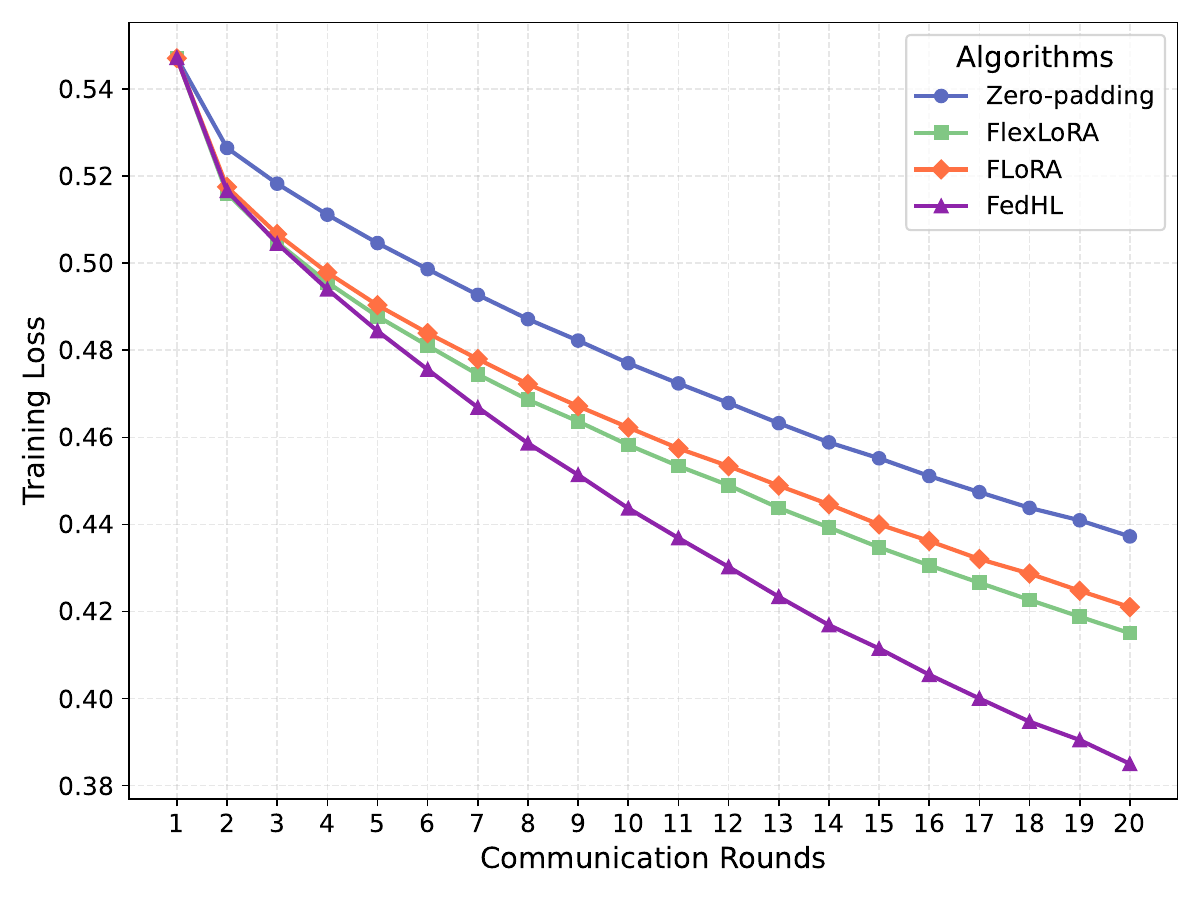}
    \end{minipage}
    \hfill
    \begin{minipage}[b]{0.48\columnwidth}
        \centering
        \includegraphics[width=\textwidth]{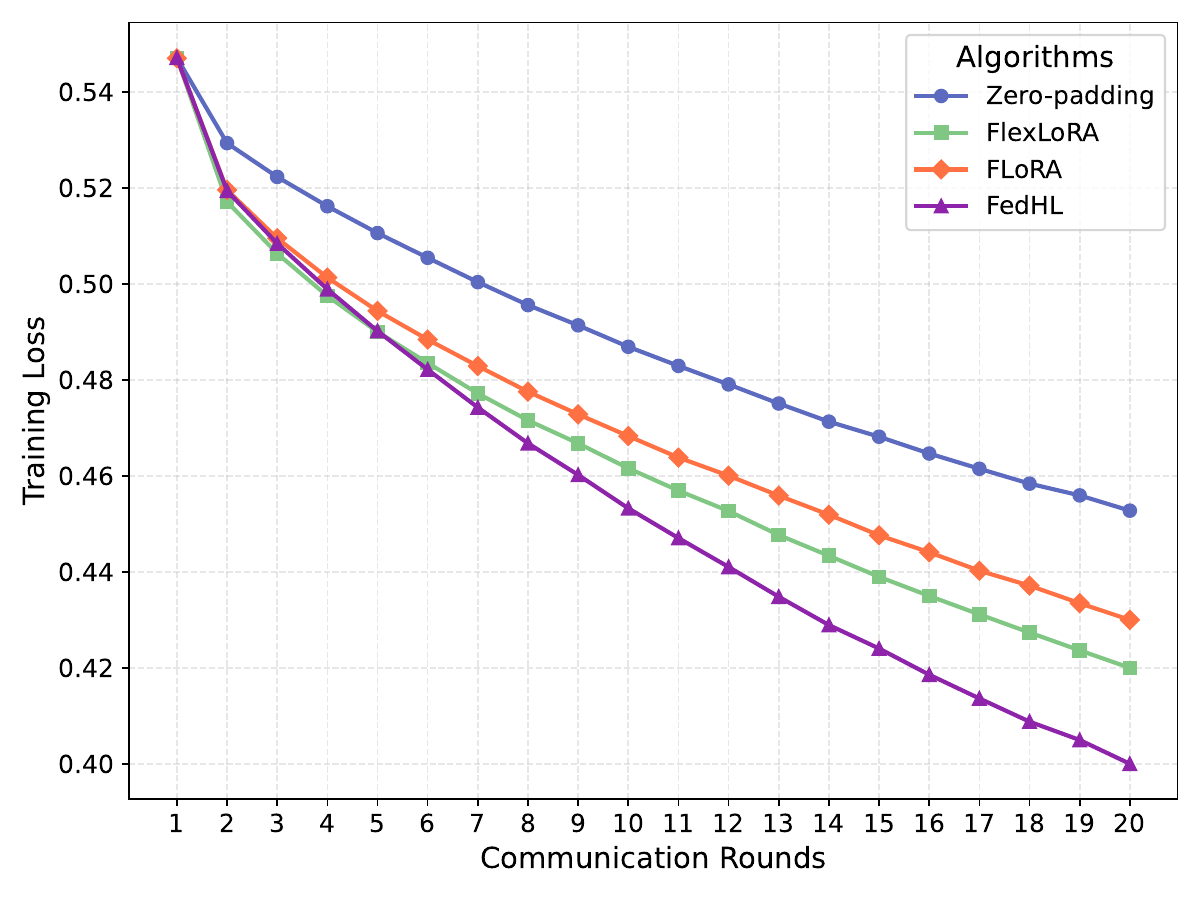}
    \end{minipage}
    \caption{Training Loss for LoRA-Based Configurations under Moderate and High Heterogeneity (Fed-CodeAlpaca).}
    \label{fig:loss_code_alpaca_100_clients}
\end{figure}
\setcounter{figure}{9}
\begin{figure}[!]
    \centering
    \begin{minipage}[b]{0.5\columnwidth}
        \centering
        \includegraphics[width=\textwidth]{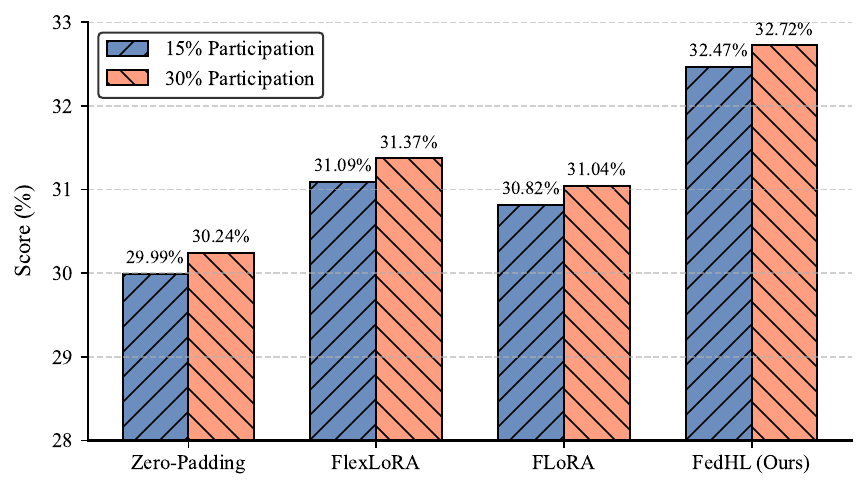}
    \end{minipage}%
    \hfill
    \begin{minipage}[b]{0.5\columnwidth}
        \centering
        \includegraphics[width=\textwidth]{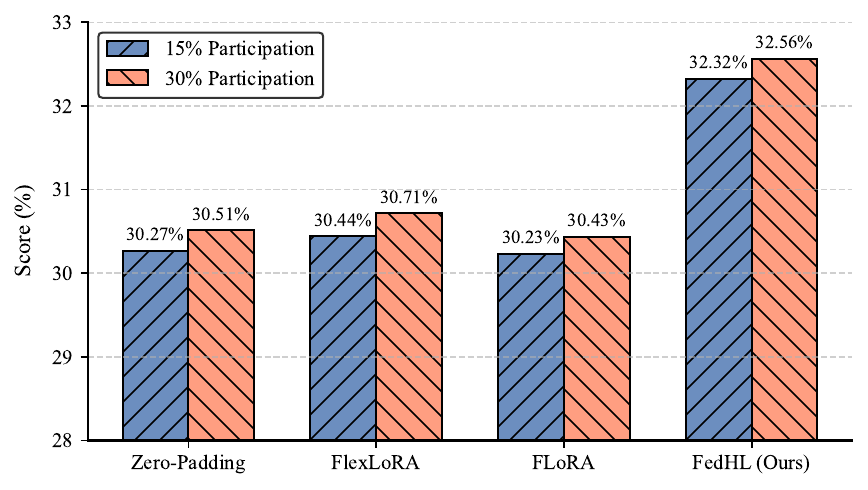}
    \end{minipage}
    \caption{Performance comparison under client participation rates (15\% vs. 30\%) on the Fed-Dolly dataset for  different rank levels: Left for moderate, Right for high.}
    \label{fig:compare_dolly}
\end{figure}

\begin{table}[!]
\caption{Cross-Device FL Performance Comparison under \underline{\textbf{Homogeneous}} LoRA Rank Settings.}

\label{tab:homogeneous_cross_device}
\centering
\setlength{\tabcolsep}{16pt}  
\renewcommand{\arraystretch}{1.2}  
\begin{tabular}{@{}lcc@{}}
\toprule

\textbf{Method} & \textbf{Fed-CodeAlpaca} & \textbf{Fed-Dolly} \\
\midrule
FedIT & $17.05^{\pm 0.32}$ & $30.20^{\pm 0.30}$ \\
Zero-Padding & $17.02^{\pm 0.34}$ & $30.15^{\pm 0.29}$ \\
FlexLoRA & $17.35^{\pm 0.25}$ & $30.78^{\pm 0.24}$ \\
FLoRA & $17.18^{\pm 0.28}$ & $30.51^{\pm 0.24}$ \\
\textbf{FedHL} & $\textbf{17.52}^{\pm 0.12}$ & $\textbf{30.93}^{\pm 0.11}$ \\
\bottomrule
\end{tabular}
\vspace{-3mm}
\end{table}
\noindent \textbf{Comparison under Homogeneous Rank Settings.} To comprehensively validate algorithmic adaptability, we conduct supplementary experiments with homogeneous LoRA ranks (fixed rank=16) under cross-device FL scenarios.  Notably, although FedHL is not specifically optimized for homogeneous configurations, it achieves the best performance across all methods in this scenario, as shown in Table~\ref{tab:homogeneous_cross_device}. By caching the global model from the previous round, FedHL retains certain high-rank information even under homogeneous settings, which contributes to improved performance. While all methods exhibit comparable performance (within 0.3\% variation) in this non-optimized scenario, FedHL demonstrates superior stability with minimal standard deviation (0.12\% vs. 0.11\% for FlexLoRA), maintaining stable performance under different rank configurations. The above performance further underscores the robustness of FedHL in diverse settings.

\subsection{Ablation experiment} 
\label{app:ablation}
\begin{table}[!]
\centering
\caption{Ablation Results: Comparison of Aggregation Strategies.}
\setlength{\tabcolsep}{8pt} 
\renewcommand{\arraystretch}{1.2}  %
\begin{tabular}{lcccc}
\toprule
\multirow{2}{*}{Aggregation Method} & \multicolumn{2}{c}{\textbf{Fed-CodeAlpaca}} & \multicolumn{2}{c}{\textbf{Fed-Dolly}} \\
\cmidrule(lr){2-3} \cmidrule(lr){4-5} 
                                    & Moderate & High & Moderate & High \\
\midrule
UA + FedAvg weight            & 20.43 & 19.75 & 32.45 & 32.36 \\
UA + Proportional        & 20.55 & 19.81 & 32.53 & 32.44 \\
UA + FedHL weight        & \textbf{20.97} & \textbf{20.29} & \textbf{32.65} & \textbf{32.63} \\
\bottomrule
\end{tabular}
\label{table:agg}
\end{table}
\noindent \textbf{Performance Comparison with Different Weighted Aggregation Methods under Unbiased Aggregation}. In this ablation study, we investigate the impact of different weighted aggregation methods using the experimental setup described in \textit{Case 1}. Specifically, we utilize the proposed $W_t$ as the basis for aggregation and compare the performance across different aggregation strategies. The proportional aggregation method assigns weights based on the ranking of the client's LoRA module, while the FedAvg method applies traditional aggregation, where weights are proportional to the client's sample size. ``UA'' in the table refers to unbiased aggregation, where $W_t$ serves as the aggregation baseline. As shown in Table~\ref{table:agg}, compared to the other two methods, the aggregation method employed by FedHL effectively mitigates gradient deviation and leads to superior performance.

\section{Conclusion}
\label{sec:conclusion}
In this paper, we propose FedHL, a novel method designed for heterogeneous LoRA training and aggregation in FL. FedHL addresses the critical challenge of per-round rank realignment, which can lead to truncation bias and gradient drift. To resolve these issues, we employ a full-rank global model as a baseline for aggregating client updates and dynamically adjust aggregation weights to minimize gradient drift. Our theoretical analysis establishes rigorous convergence guarantees for FedHL, which are often lacking in prior methods. Extensive experimental evaluations demonstrate that FedHL consistently outperforms state-of-the-art methods across a variety of rank configurations and fine-tuning datasets. Our future work will optimize rank assignments to better suit diverse client conditions.
\nocite{langley00}

\section{Proofs of Convergence for FedHL}
\label{app:th2}
To facilitate the convergence analysis, we introduce an auxiliary variable $V_{t+1}$. In the standard FL setting for full-parameter fine-tuning, the global model update at the \(t\)-th round is expressed as follows:
\begin{align}
\label{ad:1}
V_{t+1} = \sum_{i \in [N]} p_i \left( W_t - \eta_t \sum_{\tau=0}^{K-1} \tilde{\nabla} f_i(W_{t,\tau}^i) \right),
\end{align}
where \(p_i\) is the aggregate weight of client \(i\) and \(\tilde{\nabla} f_i(W_{t,\tau}^i)\) represents the stochastic gradient of client \(i\)'s model \(W_{t,\tau}^i\) at local epoch \(\tau\). Here, \(W_{t,\tau}^i\) corresponds to the full-rank fine-tuned model without truncation.

In the FedHL algorithm, the global model update at the \(t\)-th round is formulated as:
\begin{align}
\label{ad:3}
W_{t+1} =  \sum_{i \in [N]} p_i \left( W_t - \eta_t \sum_{\tau=0}^{K-1} \tilde{\nabla} f_i(W_{t,\tau}^{r_i}) \right),
\end{align}
where \(\tilde{\nabla} f_i(W_{t,\tau}^{r_i})\) represents the stochastic gradient of client \(i\)'s model \(W_{t,\tau}^{r_i}\) at local epoch \(\tau\). Unlike the standard full-parameter fine-tuning model described above, \(W_{t,\tau}^{r_i}\) represents the model truncated according to the rank \(r_i\). Before presenting the detailed proof of Theorem \ref{theorem:1}, we begin by solving these two lemmas.

\begin{lem}
\label{lemma1} We first analyze the static error introduced by the heterogeneous LoRA algorithm in the global model \(W_{t+1}\) at round \(t\) of FedHL, compared to the full-rank standard global model \(V_{t+1}\) obtained via standard gradient descent. Let the coefficient of the exponential term be \(D_0 = 4(1 + L^2 \eta_t^2) \geq 4\). Based on this, we derive the following expression for the expected difference in the loss functions between the models \(W_{t+1}\) and \(V_{t+1}\): 
\begin{align}
\label{eq:lemma1}
&\mathbb{E} f(W_{t+1}) \leq \mathbb{E} f(V_{t+1}) + \frac{(L+1)}{2} \Big[6 \eta_t^2 K^2 N \sigma_l^2 \sum_{i \in [N]} p_i^2 \notag \\  & \quad+ 3L^2 \eta_t^2 K^2 N D_0^K \sum_{i \in [N]} (p_i)^2 (\hat{r}_i)^2  \notag \\ & \quad + 24L^2 \eta_t^4 K^2 N \frac{D_0^K - 1}{D_0 - 1} \sigma_l^2 \sum_{i \in [N]} p_i^2 \Big].
\end{align}

\end{lem}
Next, we proceed with the proof of Lemma \ref{lemma1}. \begin{proof} To begin, we leverage the \(L\)-Lipschitz continuity of the gradient, as stated in Assumption \ref{assumption:1}, to expand Eqn. (\ref{ad:1}) as follows:
\begin{align}
\label{eq:B_1}
\mathbb{E}f(W_{t+1})   
&\leq \mathbb{E}f(V_{t+1}) + \underbrace{ \mathbb{E}\langle \nabla f(V_{t+1}), W_{t+1} - V_{t+1} \rangle}_{T_1} \notag \\ &+ \frac{L}{2}\underbrace{\mathbb{E}\|W_{t+1} - V_{t+1}\|^2}_{T_2}. 
\end{align}
Next, we proceed to bound the terms \(T_1\) and \(T_2\) individually.

\begin{align}
\label{bound_t2}
T_2 = &\ \mathbb{E}\left\| W_{t+1} - V_{t+1} \right\|^2 \notag \\
= &\ \mathbb{E} \left\| \sum_{i \in [N]} p_i \sum_{\tau=0}^{K-1} \eta_t \left[ \tilde{\nabla} f_i(W_{t,\tau}^i) - \tilde{\nabla} f_i(W_{t,\tau}^{r_i}) \right] \right\|^2 \notag \\
= &\ \eta_t^2 \mathbb{E} \left\| \sum_{i \in [N]} p_i \sum_{\tau=0}^{K-1} \left\{ \left[ \nabla f_i(W_{t,\tau}^i) - \nabla f_i(W_{t,\tau}^{r_i}) \right] \right. \right. \notag \\
&\quad + \left[ \tilde{\nabla} f_i(W_{t,\tau}^i) - \nabla f_i(W_{t,\tau}^i) \right] \notag \\
&\quad \left. \left. - \left[ \tilde{\nabla} f_i(W_{t,\tau}^{r_i}) - \nabla f_i(W_{t,\tau}^{r_i}) \right] \right\} \right\|^2 \notag \\
\overset{(a)}{=} &\ 3\eta_t^2 \mathbb{E} \left\| \sum_{i \in [N]} p_i \sum_{\tau=0}^{K-1} \left[ \nabla f_i(W_{t,\tau}^i) - \nabla f_i(W_{t,\tau}^{r_i}) \right] \right\|^2 \notag \\
&\ + 3 \eta_t^2 \mathbb{E} \left\| \sum_{i \in [N]} p_i \sum_{\tau=0}^{K-1} \left[ \tilde{\nabla} f_i(W_{t,\tau}^i) - \nabla f_i(W_{t,\tau}^i) \right] \right\|^2 \notag \\
&\ + 3\eta_t^2 \mathbb{E} \left\| \sum_{i \in [N]} p_i \sum_{\tau=0}^{K-1} \left[ \tilde{\nabla} f_i(W_{t,\tau}^{r_i}) - \nabla f_i(W_{t,\tau}^{r_i}) \right] \right\|^2 \notag \\
\overset{(b)}{\leq} &\ 3 \eta_t^2 N K \sum_{i \in [N]} p_i^2 \sum_{\tau=0}^{K-1} \mathbb{E} \left\| \nabla f_i(W_{t,\tau}^i) - \nabla f_i(W_{t,\tau}^{r_i}) \right\|^2 \notag \\
&\ + 6 \eta_t^2 N K^2 \sigma_l^2 \sum_{i \in [N]} p_i^2 \notag \\
\overset{(c)}{\leq} &\ 3 L^2 \eta_t^2 N K \sum_{i \in [N]} p_i^2 \sum_{\tau=0}^{K-1} \underbrace{ \mathbb{E} \left\| W_{t,\tau}^i - W_{t,\tau}^{r_i} \right\|^2 }_{T_3} \notag \\ & + 6 \eta_t^2 N K^2 \sigma_l^2 \sum_{i \in [N]} p_i^2. 
\end{align}

Equality (a) follows from the Cauchy-Schwarz inequality. Inequality (b) is a consequence of Assumption \ref{assumption:2}, which bounds the local gradient variance, while inequality (c) is derived using the \(L\)-Lipschitz continuity condition stated in Assumption \ref{assumption:1}.

During the investigation of \(T_2\), the term \(\mathbb{E}\|W_{t,\tau}^i - W_{t,\tau}^{r_i}\|^2\) emerges, representing the expected difference between the full-rank standard gradient descent model and the heterogeneous rank gradient descent model. We denote this term as \(T_3\) and proceed to derive an upper bound for \(T_3 = \Gamma_\tau\) in the following analysis. We proceed to bound \( \Gamma_\tau \) as follows:

\begin{align}
\Gamma_0 &= \mathbb{E}\left\| W_{t,0}^i - W_{t,0}^{r_i} \right\|^2 = \mathbb{E}\left\| W_t - W_t^{r_i} \right\|^2 \leq \hat{r}_i^2, \\ \notag
\Gamma_1 &= \mathbb{E}\Big\| W_{t,0}^i - W_{t,0}^{r_i} - \eta_t \left[ \nabla f_i(W_{t,0}^i) - \nabla f_i(W_{t,0}^{r_i}) \right]  \\ \notag
&\quad - \eta_t \left[ \tilde{\nabla} f_i(W_{t,0}^i) - \nabla f_i(W_{t,0}^i) \right] \\ \notag & \quad + \eta_t \left[ \tilde{\nabla} f_i(W_{t,0}^{r_i}) - \nabla f_i(W_{t,0}^{r_i}) \right] \Big\|^2  \\ \notag
&\overset{(a)}{\leq} 4\, \mathbb{E}\left\| W_{t,0}^i - W_{t,0}^{r_i} \right\|^2 + 4\, \eta_t^2 \mathbb{E}\left\| \nabla f_i(W_{t,0}^i) - \nabla f_i(W_{t,0}^{r_i}) \right\|^2 \\ \notag  & \quad+ 8 \eta_t^2 \sigma_l^2 \label{eq:Gamma1} \\
&\overset{(b)}{\leq} 4\, \Gamma_0 + 4\, L^2 \eta_t^2 \Gamma_0 + 8\, \eta_t^2 \sigma_l^2 \\  \notag 
\Gamma_2 &\leq 4(1 + L^2 \eta_t^2) \Gamma_1 + 8\, \eta_t^2 \sigma_l^2 \notag \\
&\leq 4(1 + L^2 \eta_t^2) \left[ 4(1 + L^2 \eta_t^2) \Gamma_0 + 8\, \eta_t^2 \sigma_l^2 \right] + 8\, \eta_t^2 \sigma_l^2 \notag \\
&= \left[ 4(1 + L^2 \eta_t^2) \right]^2 \Gamma_0 + \sum_{u=0}^{1} \left[ 4(1 + L^2 \eta_t^2) \right]^u \cdot 8\, \eta_t^2 \sigma_l^2, \\
&\ \vdots \notag \\
\Gamma_{\tau} &\leq \left[ 4(1 + L^2 \eta_t^2) \right]^{\tau} \Gamma_0 + \sum_{u=0}^{\tau - 1} \left[ 4(1 + L^2 \eta_t^2) \right]^u \cdot 8\, \eta_t^2 \sigma_l^2. \label{eq:GammaTau}
\end{align}

Inequality (a) is derived by applying the Cauchy-Schwarz inequality together with Assumption \ref{assumption:2}. Inequality (b) follows from the definition of \(\Gamma_0\) and the \(L\)-Lipschitz continuity condition outlined in Assumption \ref{assumption:1}.

By recursively applying these bounds, we establish that for any $\tau \geq 0$,
\begin{align}
\Gamma_{\tau} \leq \left[4(1 + L^2\eta_t^2)\right]^{\tau} \Gamma_0 + \sum_{u=0}^{\tau - 1} \left[4(1 + L^2\eta_t^2)\right]^u 8 \eta_t^2 \sigma_l^2.
\end{align}
Let $D_0 = 4(1 + L^2\eta_t^2) \geq 4$. Then, $\Gamma_\tau$ is bounded by the following expression:
\begin{align}
\label{t3_bound}
\Gamma_\tau \leq D_0^K (\hat{r}_i)^2 + 8\frac{D_0^K - 1}{D_0 - 1}\eta_t^2\sigma_l^2.
\end{align}
Next, we substitute the upper bound derived from Inequality (\ref{t3_bound}) into (\ref{bound_t2}), yielding the following expression for the upper bound of $T_2$:

\begin{align}
\label{eq:T_2}
T_2 \leq & 3L^2 \eta_t^2 N K \sum_{i \in [N]} p_i^2 \sum_{\tau=0}^{K-1} \left[D_0^K (\hat{r}_i)^2 + 8\frac{D_0^K-1}{D_0-1}\eta_t^2 \sigma_l^2 \right] \notag \\  &+ 6 \eta_t^2 N K^2 \sigma_l^2 \sum_{i \in [N]} p_i^2 \notag \\
 \leq & 3L^2 \eta_t^2 N K^2 D_0^K \sum_{i \in [N]} (p_i)^2 (\hat{r}_i)^2 + 6 \eta_t^2 N K^2 \sigma_l^2 \sum_{i \in [N]} p_i^2 \notag \\  &+ 24 L^2 \eta_t^4 N K^2 \frac{D_0^K-1}{D_0-1} \sigma_l^2 \sum_{i \in [N]} p_i^2.
\end{align}

We now proceed to analyze the upper bound of the term \(T_1\).
\begin{align}
\label{eq:T_1}
T_1 &= \mathbb{E} \langle \nabla f(V_{t+1}),\, W_{t+1} - V_{t+1} \rangle \notag \\
    &\overset{(a)}{=} \frac{1}{2} \left[ \mathbb{E} \| \nabla f(V_{t+1}) \|^2 + \mathbb{E} \| W_{t+1} - V_{t+1} \|^2 \right. \notag \\
    &\quad \left. - \mathbb{E} \| \nabla f(V_{t+1}) - (W_{t+1} - V_{t+1}) \|^2 \right] \notag \\
    &\overset{(b)}{\leq} \frac{1}{2} \mathbb{E} \| \nabla f(V_{t+1}) \|^2 + \frac{1}{2} \mathbb{E} \| W_{t+1} - V_{t+1} \|^2 \notag \\
    &\overset{(c)}{\lesssim} \frac{1}{2} \mathbb{E} \| W_{t+1} - V_{t+1} \|^2.
\end{align}

Equality (a) follows from the identity $
2\langle a, b \rangle = \| a \|^2 + \| b \|^2 - \| a - b \|^2,$
applied with \(a = \nabla f(V_{t+1})\) and \(b = W_{t+1} - V_{t+1}\). Inequality (b) holds as the term \(
- \frac{1}{2} \mathbb{E} \left\| \nabla f(V_{t+1}) - (W_{t+1} - V_{t+1}) \right\|^2 \)
is non-positive. Inequality (c) is supported by the fact that, in standard FedAvg, $\mathbb{E} \left\| \nabla f(V_{t+1}) \right\|^2$ decreases over time \cite{stich2019local, yu2019parallel}. As $t \to \infty$, this term asymptotically approaches zero. For brevity, further details are omitted as they do not change the convergence behavior of federated learning for heterogeneous LoRA. By integrating the results from (\ref{eq:T_1}) and (\ref{eq:T_2}) into (\ref{eq:B_1}), we derive the upper bound for \( \mathbb{E} f(W_{t+1}) \), as stated in Lemma \ref{lemma1}.
\end{proof}

\begin{lem}
\label{lemma2}
We establish the relationship between $V_{t+1}$ and $W_{t+1}$ under the conditions $\frac{1}{2} - 9L^2 \eta_t^2 K^2 N \sum_{i \in [N]} p_i^2 \geq C > 0$ and $\frac{L \eta_t^2 K}{2} - \frac{\eta_t}{2} \leq 0$. Given these conditions, the expected difference between the loss functions of $V_{t+1}$ and $W_{t+1}$ is expressed as follows:
\end{lem}

\begin{align}
\label{eq:part2}
& \mathbb{E} f(V_{t+1}) \leq  \mathbb{E} f(W_t) - C \eta_t K \mathbb{E} \|\nabla f(W_t)\|^2   \notag \\ & \quad + \frac{3}{2} L^2 \eta_t^3 K^2 N (\sigma_l^2 + 6K \sigma_g^2) \sum_{i \in [N]} p_i^2 +\frac{L \eta_t^2 K^2 N}{2} \sigma_l^2 \sum_{i \in [N]} p_i^2.
\end{align}

Next, the proof of Lemma \ref{lemma2} is outlined below. 
\begin{proof}
To begin, by applying the L-Lipschitz continuity condition for gradients as stated in Assumption \ref{assumption:1}, we expand Eqn. (\ref{ad:3}) in the following manner:

\begin{align}
\label{eq:B}
\mathbb{E} f(V_{t+1}) &= \ \mathbb{E} f ( W_t - \sum_{i \in [N]} p_i \sum_{\tau=0}^{K-1} \eta_t \tilde{\nabla} f_i(W_{t,\tau}^i) ) 
\notag \\ \leq &\ \mathbb{E} f(W_t) + \underbrace{\mathbb{E} \langle \nabla f(W_t), -\sum_{i \in [N]} p_i \sum_{\tau=0}^{K-1} \eta_t \tilde{\nabla} f_i(W_{t,\tau}^i) \rangle}_{T_4}  \notag \\ &+ \underbrace{\frac{L}{2} \mathbb{E} \| \sum_{i \in [N]} p_i \sum_{\tau=0}^{K-1} \eta_t \tilde{\nabla} f_i(W_{t,\tau}^i) \|^2}_{T_5}.
\end{align}

Next, we proceed by analyzing the two key components, \( T_4 \) and \( T_5 \), separately. 

\begin{align}
\label{eq:T_4}
T_4 =&\ \mathbb{E} \langle \nabla f(W_t), - \sum_{i \in [N]} p_i \sum_{\tau=0}^{K-1} \eta_t \tilde{\nabla} f_i(W_{t,\tau}^i) \rangle \notag \\ \overset{(a)}{=}&\ \mathbb{E} \langle \nabla f(W_t), - \sum_{i \in [N]} p_i \sum_{\tau=0}^{K-1} \eta_t \nabla f_i(W_{t,\tau}^i) \rangle \notag \\
 \overset{(b)}{\leq}&\ - \eta_t \sum_{\tau=0}^{K-1} [ \frac{1}{2} \mathbb{E} \|\nabla f(W_t)\|^2  + \frac{1}{2} \mathbb{E} \|\sum_{i \in [N]} p_i \nabla f_i(W_{t,\tau}^i) \|^2 \notag \\ & - \frac{1}{2} \|\nabla f(W_t) - \sum_{i \in [N]} p_i \nabla f_i(W_{t,\tau}^i)\|^2] \notag \\
 \overset{(c)}{\leq}&\ - \eta_t \sum_{\tau=0}^{K-1} [ \frac{1}{2} \mathbb{E} \|\nabla f(W_t)\|^2 + \frac{1}{2} \mathbb{E} \|\sum_{i \in [N]} p_i \nabla f_i(W_{t,\tau}^i) \|^2 \notag  \\ & - \frac{1}{2} \|\sum_{i \in [N]} p_i (\nabla f_i(W_t) - \nabla f_i(W_{t,\tau}^i))\|^2] \notag \\
 \overset{(d)}{\leq}&\ - \frac{\eta_t K}{2} \mathbb{E} \|\nabla f(W_t)\|^2  - \frac{\eta_t}{2} \sum_{\tau=0}^{K-1} \mathbb{E} \|\sum_{i \in [N]} p_i \nabla f_i(W_{t,\tau}^i) \|^2 \notag \\ &+ \frac{\eta_t N}{2} \sum_{i \in [N]} p_i^2 \sum_{\tau=0}^{K-1} \|\nabla f_i(W_t) - \nabla f_i(W_{t,\tau}^i)\|^2 \notag \\
 \overset{(e)}{\leq}&\ - \frac{\eta_t K}{2} \mathbb{E} \|\nabla f(W_t)\|^2  - \frac{\eta_t}{2} \sum_{\tau=0}^{K-1} \mathbb{E} \|\sum_{i \in [N]} p_i \nabla f_i(W_{t,\tau}^i) \|^2 \notag \\ &+ \frac{L^2 \eta_t N}{2} \sum_{i \in [N]} p_i^2 \sum_{\tau=0}^{K-1} \mathbb{E} \underbrace{\|W_t - W_{t,\tau}^i\|^2}_{T_6}.
\end{align}

Equality (a) holds due to Assumption \ref{assumption:2}. Inequality (b) follows from the identity \( 2 \langle a, b \rangle = \|a\|^2 + \|b\|^2 - \|a - b\|^2 \). Inequality (c) holds due to the definition of the FL loss function. Inequality (d) is established by applying the Cauchy-Schwarz inequality. Finally, inequality (e) holds due to the L-Lipschitz condition from Assumption \ref{assumption:1}. 

Next, we focus on analyzing \( T_6 \), which represents the expected squared difference between \( W_t \) and \( W_{t,\tau}^i \), which captures the discrepancy between the global model and the local models during the iterative updates.

\begin{align}
T_6 
= & \mathbb{E} \| W_{t,\tau-1}^i - W_t - \eta_t \Tilde{\nabla} f_i(W_{t,\tau-1}^i) \|^2 \notag \\
\overset{(a)}{=} & \mathbb{E} \| W_{t,\tau-1}^i - W_t - \eta_t ( \nabla f_i(W_{t,\tau-1}^i) - \nabla f_i(W_t) + \nabla f_i(W_t) \notag \\
& - \nabla f(W_t) + \nabla f(W_t) ) \|^2 \notag \\
& + \mathbb{E} \| \eta_t ( \Tilde{\nabla} f_i(W_{t,\tau-1}^i) - \nabla f_i(W_{t,\tau-1}^i) ) \|^2 \notag \\
\overset{(b)}{=} & ( 1 + \frac{1}{2K-1}) \mathbb{E} \left\| W_{t,\tau-1}^i - W_t \right\|^2  + 2 \eta_t^2 K \mathbb{E} \| \nabla f_i(W_{t,\tau-1}^i) \notag \\
&- \nabla f_i(W_t) + \nabla f_i(W_t)  - \nabla f(W_t) + \nabla f(W_t) \|^2 \notag \\
& + \mathbb{E} \left\| \eta_t \left( \Tilde{\nabla} f_i(W_{t,\tau-1}^i) - \nabla f_i(W_{t,\tau-1}^i) \right) \right\|^2 \notag \\
\overset{(c)}{=} & ( 1 + \frac{1}{2K-1} ) \mathbb{E} \left\| W_{t,\tau-1}^i - W_t \right\|^2 + \eta_t^2 \sigma_l^2 + 6 \eta_t^2 K \sigma_g^2 \notag \\
&+ 6 \eta_t^2 K \|\nabla f(W_t)\|^2  + 6 \eta_t^2 K \mathbb{E} \|\nabla f_i(W_{t,\tau-1}^i) - \nabla f_i(W_t)\|^2  \notag \\
\overset{(d)}{=} & ( 1 + \frac{1}{2K-1} + 6 \eta_t^2 K L^2 ) \mathbb{E} \| W_{t,\tau-1}^i - W_t \|^2 + \eta_t^2 \sigma_l^2  \notag \\
& + 6 \eta_t^2 K \sigma_g^2 + 6 \eta_t^2 K \|\nabla f(W_t)\|^2 \notag \\
\overset{(e)}{\leq} & ( 1 + \frac{1}{K-1} ) \mathbb{E} \| W_{t,\tau-1}^i - W_t \|^2 + \eta_t^2 \sigma_l^2 + 6 \eta_t^2 K \sigma_g^2 +  \notag \\
& 6 \eta_t^2 K \|\nabla f(W_t)\|^2. 
\end{align}

In step (a), Assumption \ref{assumption:2} is applied. Step (b) follows from the inequality \( \|a+b\|^2 \leq (1 + \frac{1}{k})\|a\|^2 + (k+1)\|b\|^2 \) for all vectors \( a \) and \( b \), where \( k = 2K - 1 \). Step (c) is derived using the Cauchy-Schwarz inequality, as well as Assumptions \ref{assumption:1} and \ref{assumption:3}. Step (d) holds based on the L-Lipschitz condition in Assumption \ref{assumption:1}. Finally, under the conditions specified in Theorem \ref{theorem:1}, 
in particular, with a sufficiently small learning rate, inequality (e) is satisfied.

Let \(\mathbb{E}\|W_{t,\tau}^i - W_t\|^2 = \Delta_\tau\), where \(\Delta_\tau\) represents the expected deviation between the local model at step \(\tau\) and the global model \(W_t\). We establish the following recursive inequality:
\begin{align}
\label{bound_t6}
\Delta_\tau &\leq \left(1 + \frac{1}{K - 1}\right) \Delta_{\tau - 1} + \eta_t^2 \sigma_l^2 + 6\eta_t^2 K \sigma_g^2 \notag \\
& \quad + 6\eta_t^2 K \mathbb{E}\|\nabla f(W_t)\|^2 \notag \\
&\leq \left(1 + \frac{1}{K - 1}\right) \left[\left(1 + \frac{1}{K - 1}\right) \Delta_{\tau - 2} + \eta_t^2 \sigma_l^2 + 6\eta_t^2 K \sigma_g^2 \right. \notag \\
& \quad \left. + 6\eta_t^2 K \mathbb{E}\|\nabla f(W_t)\|^2 \right] + \eta_t^2 \sigma_l^2 + 6\eta_t^2 K \sigma_g^2 \notag \\
& \quad + 6\eta_t^2 K \mathbb{E}\|\nabla f(W_t)\|^2 \notag \\
&\leq \left(1 + \frac{1}{K - 1}\right)^\tau \Delta_0 + \sum_{j=0}^{\tau - 1} \left(1 + \frac{1}{K - 1}\right)^j \left[\eta_t^2 \sigma_l^2 + 6\eta_t^2 K \sigma_g^2 \right. \notag \\
& \quad \left. + 6\eta_t^2 K \mathbb{E}\|\nabla f(W_t)\|^2 \right] \notag \\
&= \left[\left(1 + \frac{1}{K - 1}\right)^\tau - 1\right] (K - 1) \left[\eta_t^2 \sigma_l^2 + 6\eta_t^2 K \sigma_g^2 \right. \notag \\
& \quad \left. + 6\eta_t^2 K \mathbb{E}\|\nabla f(W_t)\|^2 \right] \notag \\
&\leq \left[\left(1 + \frac{1}{K - 1}\right)^K - 1\right] (K - 1) \left[\eta_t^2 \sigma_l^2 + 6\eta_t^2 K \sigma_g^2 \right. \notag \\
& \quad \left. + 6\eta_t^2 K \mathbb{E}\|\nabla f(W_t)\|^2\right] \notag \\
&\leq 3 \eta_t^2 K (\sigma_l^2 + 6K \sigma_g^2) + 18 \eta_t^2 K^2 \mathbb{E}\|\nabla f(W_t)\|^2.
\end{align}

By substituting Equation \eqref{bound_t6} into Equation \eqref{eq:T_4}, we obtain the upper bound of $T_4$ as follows:
\begin{align}
\label{bound_t4}
T_4 \leq & -(\frac{\eta_t K}{2} - 9L^2 \eta_t^3 K^3 N \sum_{i \in [N]} p_i^2) \mathbb{E}\|\nabla f(W_t)\|^2 \notag \\  &+ \frac{3}{2}L^2 \eta_t^3 K^2 N (\sigma_l^2 + 6K\sigma_g^2) \sum_{i \in [N]} p_i^2 \notag \\  & - \frac{\eta_t}{2} \sum_{\tau=0}^{K-1} \mathbb{E}\|\sum_{i \in [N]} p_i \nabla f_i(W_{t,\tau}^i)\|^2.
\end{align}

Next, we analyze the upper bound of $T_5$.
\begin{align}
\label{eq:T_5}
T_5 =&\ \frac{L}{2} \mathbb{E} \| \sum_{i \in [N]} p_i \sum_{\tau=0}^{K-1} \eta_t \tilde{\nabla} f_i(W_{t,\tau}^i) \|^2 \notag 
 \\ 
\overset{(a)}{=}&\ \frac{L}{2} \mathbb{E} \left\| \sum_{i \in [N]} p_i \sum_{\tau=0}^{K-1} \eta_t \left[ \tilde{\nabla} f_i(W_{t,\tau}^i) - \nabla f_i(W_{t,\tau}^i) \right] \right\|^2 \notag \\ &+ \frac{L}{2} \mathbb{E} \left\| \sum_{i \in [N]} p_i \sum_{\tau=0}^{K-1} \eta_t \nabla f_i(W_{t,\tau}^i) \right\|^2 \notag \\
 \overset{(b)}{\leq}&\ \frac{L \eta_t^2 K N}{2} \sum_{i \in [N]} p_i^2 \sum_{\tau=0}^{K-1} \mathbb{E} \left\| \tilde{\nabla} f_i(W_{t,\tau}^i) - \nabla f_i(W_{t,\tau}^i) \right\|^2 \notag \\ & + \frac{L \eta_t^2 K}{2} \sum_{\tau=0}^{K-1} \mathbb{E} \left\| \sum_{i \in [N]} p_i \nabla f_i(W_{t,\tau}^i) \right\|^2 \notag \\
 \overset{(c)}{\leq}&\ \frac{L \eta_t^2 K^2 N}{2} \sigma_l^2 \sum_{i \in [N]} p_i^2 + \frac{L \eta_t^2 K}{2} \sum_{\tau=0}^{K-1} \mathbb{E} \| \sum_{i \in [N]} p_i \nabla f_i(W_{t,\tau}^i) \|^2.
\end{align}

where (a) holds due to the unbiasedness of local gradients from Assumption (\ref{assumption:2}), (b) is valid by the Cauchy-Schwarz inequality, and (c) holds because of the bounded variance of local gradients in Assumption (\ref{assumption:2}).

Substituting the results derived from $T_4$ and $T_5$ into equation \label{eq
}, under the conditions $\frac{1}{2} - 9L^2\eta_t^2K^2N\sum_{i \in [N]}p_i^2 \geq C > 0$ and $\frac{L \eta_t^2 K}{2} - \frac{\eta_t}{2} \leq 0$, we obtain the following result:
\begin{align}
\label{eq:part2}
&\mathbb{E}f(V_{t+1}) \leq  \mathbb{E} f(W_t) -\eta_tK(\frac{1}{2} -9L^2\eta_t^2K^2N\sum_{i \in [N]}p_i^2) * \notag \\ & \quad\mathbb{E}\|\nabla f(W_t)\|^2 + \frac{3}{2}L^2\eta_t^3K^2N(\sigma_l^2 + 6K\sigma_g^2) \sum_{i\in[N]} p_i^2  \notag \\&+ (\frac{L \eta_t^2 K}{2}-\frac{\eta_t}{2}) \sum_{\tau=0}^{K-1} \mathbb{E}\|\sum_{i \in [N]}p_i \nabla f_i(W_{t,\tau}^i)\|^2   \notag \\  &+  \frac{L \eta_t^2 K^2 N}{2} \sigma_l^2 \sum_{i \in [N]} p_i \notag  \\ \leq & \mathbb{E}f(W_t) - C \eta_tK \mathbb{E}\|\nabla f(W_t)\|^2 +  \frac{L \eta_t^2 K^2 N}{2} \sigma_l^2 \sum_{i \in [N]} p_i^2 \notag \\ &+ \frac{3}{2}L^2\eta_t^3K^2N(\sigma_l^2 + 6K\sigma_g^2) \sum_{i\in[N]}  p_i^2 .
\end{align}
This completes the proof of Lemma \ref{lemma2}.
\end{proof}

\begin{proof}
Based on the results proven in the previous two lemmas, corresponding to lemma \ref{lemma1} and lemma \ref{lemma2}, we combine them to establish the relationship between the expressions for $f(W_{t+1})$ and $f(W_{t})$.
\begin{align}
\label{th1:proof1}
 &\mathbb{E} f(W_{t+1}) \leq\ \mathbb{E} f(W_t) - C \eta_t K \mathbb{E} \|\nabla f(W_t)\|^2  
 \notag \\ &+ \frac{L \eta_t^2 K^2 N}{2} \sigma_l^2 \sum_{i \in [N]} p_i^2 +\frac{3}{2} L^2 \eta_t^3 K^2 N (\sigma_l^2 + 6 K \sigma_g^2) \sum_{i \in [N]} p_i^2 \notag \\ &+\frac{(L+1)}{2} [ 3 L^2 \eta_t^2 K^2 N D_0^K  \sum_{i \in [N]} p_i^2 (\hat{r}_i)^2 + 6 \eta_t^2 K^2 N \sigma_l^2 \sum_{i \in [N]} p_i^2 \notag \\ & + 24 L^2 \eta_t^4 K^2 N \frac{D_0^K - 1}{D_0 - 1} \sigma_l^2 \sum_{i \in [N]} p_i^2].
\end{align}

We sum both sides of Eqn. (\ref{th1:proof1}) over $t = 0, \dots , T-1$, and rearrange the terms.

\begin{align}
& \sum_{t=0}^{T-1} C \eta_t K \mathbb{E}\|\nabla f(W_t)\|^2 \leq f(W_0) - f(W_T) \notag \\
& + \frac{(L+1)}{2} [ 
    3L^2 K^2 N D_0^K \sum_{i \in [N]} (p_i)^2 (\hat{r}_i)^2 \sum_{t=0}^{T-1} \eta_t^2
    + 24L^2 K^2 N * \notag \\ & \quad\frac{D_0^K - 1}{D_0 - 1} \sigma_l^2 \sum_{i \in [N]} p_i^2 \sum_{t=0}^{T-1} \eta_t^4 
    + 6 K^2 N \sigma_l^2 \sum_{i \in [N]} p_i^2 \sum_{t=0}^{T-1} \eta_t^2 
] \notag \\
&  + \frac{3}{2} L^2 K^2 N (\sigma_l^2 + 6K \sigma_g^2) \sum_{i \in [N]} p_i^2 \sum_{t=0}^{T-1} \eta_t^3 
\notag \\ & \quad + \frac{L K^2 N}{2} \sigma_l^2 \sum_{i \in [N]} p_i^2 \sum_{t=0}^{T-1} \eta_t^2.
\end{align}
Let the learning rate $\eta_t = \frac{1}{L K \sqrt{T}}$, which yields
final results.
\end{proof}

\section{Proofs of Convergence for the Existing Heterogeneous LoRA Algorithm}
\label{Appendix:th2}
To facilitate the convergence analysis in the current heterogeneous LoRA algorithm, we introduce an auxiliary variable, similar to the previous analysis. The ideal global full fine-tuning model update without truncation at the \(t\)-th round is expressed as follows:
\begin{align}
V_{t+1} =  \sum_{i \in [N]} p_i \left[ Q_t - \sum_{\tau=0}^{K-1} \eta_t \tilde{\nabla} f_i(Q_{t,\tau}^i) \right].
\end{align}
In the existing heterogeneous LoRA algorithm, the global model update at the \(t\)-th round is formulated as:
\begin{align}
\label{ad:2}
Q_{t+1} =  \sum_{i \in [N]} p_i \left[ Q_t^{r_i} - \sum_{\tau=0}^{K-1} \eta_t \tilde{\nabla} f_i(Q_{t,\tau}^{r_i}) \right],
\end{align}

where \(r_i\) denotes the rank of the LoRA matrix for client \(i\), and \(\tilde{\nabla} f_i(Q_{t,\tau}^{r_i})\) represents the stochastic gradient of client \(i\) model $Q_{t,\tau}^{r_i}$ at local epoch \(\tau\).

Theorem \ref{theorem:2} explores the scenario in which the use of \textbf{$Q_t^{r_i}$}, adopted by all current state-of-the-art algorithms, as the base without employing our FedHL unbiased aggregation method leads to model convergence failure. Before proving the theorem, we first derive the relationship between \(Q_{t+1}\) and \(V_{t+1}\) in Lemma \ref{lemma3}, which mathematically explains how truncation introduces learning bias in heterogeneous  training, thereby hindering convergence.

\begin{lem}
\label{lemma3}
We first derive the static error introduced by the current state-of-the-art heterogeneous LoRA algorithms in the global model $Q_{t+1}$ at round $t$, compared to the full-rank standard global model $V_{t+1}$ obtained through standard gradient descent. The coefficient of the exponential term is $D_0 = 4(1+L^2\eta_t^2) \geq 4$. From this, we derive the following expression for the expected difference in the loss functions of the global models $Q_{t+1}$ and $V_{t+1}$:
\begin{align}
&\mathbb{E} f(Q_{t+1}) \leq \ \mathbb{E} f(V_{t+1}) + \frac{(L+1)}{2} [ 2N \sum_{i \in [N]} (p_i)^2 (\hat{r}_i)^2 \notag \\ & \quad+ 6 L^2 \eta_t^2 NK^2 D_0^K \sum_{i \in [N]} (p_i)^2 (\hat{r}_i)^2 \notag \\
&\quad + 48 L^2 \eta_t^4 NK^2 \frac{D_0^K - 1}{D_0 - 1} \sigma_l^2 \sum_{i \in [N]} p_i^2 + 12 \eta_t^2 N K^2 \sigma_l^2 \sum_{i \in [N]} p_i^2 ].
\end{align}
\end{lem}
\begin{proof}
Similarly, based on the L-Lipschitz continuity condition, we can obtain:
\begin{align}
\mathbb{E}f(Q_{t+1}) &= \mathbb{E} f(V_{t+1} + Q_{t+1} - V_{t+1}) \notag \\ & \leq \mathbb{E} f(V_{t+1}) + \mathbb{E} \left\langle \nabla f(V_{t+1}), Q_{t+1} - V_{t+1} \right\rangle \notag \\  & \quad+ \frac{L}{2} \mathbb{E} \left\| Q_{t+1} - V_{t+1} \right\|^2.
\end{align}
Next, we investigate the upper bound of $\mathbb{E}\|Q_{t+1}-V_{t+1}\|^2$.

\begin{align}
& \mathbb{E} \|Q_{t+1}-V_{t+1}\|^2 \notag \\ =& \mathbb{E} \| \sum_{i \in [N]} p_i (Q_t^{r_i} - Q_t) \notag \\ &\quad - \sum_{i \in [N]} p_i \sum_{\tau=0}^{K-1} \eta_t (\Tilde{\nabla} f_i(Q_{t,\tau}^{r_i}) - \Tilde{\nabla} f_i(Q_{t,\tau})) \|^2 \notag \\
\overset{(a)}{\leq} & 2 \mathbb{E} \| \sum_{i \in [N]} p_i (Q_t^{r_i} - Q_t) \|^2 + 2 \eta_t^2 \mathbb{E} \| \sum_{i \in [N]} p_i \sum_{\tau=0}^{K-1} \{[\nabla f_i(Q_{t,\tau}^{r_i}) \notag \\ &- \nabla f_i(Q_{t,\tau})] + [\Tilde{\nabla} f_i(Q_{t,\tau}^{r_i}) - \nabla f_i(Q_{t,\tau}^{r_i})] \notag \\
& - [\Tilde{\nabla} f_i(Q_{t,\tau}) - \nabla f_i(Q_{t,\tau})]\} \|^2 \notag \\
\overset{(b)}{\leq} & 2N \sum_{i \in [N]} p_i^2 \mathbb{E} \| Q_t^{r_i} - Q_t \|^2 \notag \\ &+ 6 \eta_t^2 NK \sum_{i \in [N]} p_i^2 \mathbb{E} \sum_{\tau=0}^{K-1} \| \nabla f_i(Q_{t,\tau}^{r_i}) - \nabla f_i(Q_{t,\tau}) \|^2 \notag \\
& + 6 \eta_t^2 NK \sum_{i \in [N]} p_i^2 \sum_{\tau=0}^{K-1} \mathbb{E} \| \Tilde{\nabla} f_i(Q_{t,\tau}^{r_i}) - \nabla f_i(Q_{t,\tau}^{r_i}) \|^2 \notag \\
& + 6 \eta_t^2 NK \sum_{i \in [N]} p_i^2 \sum_{\tau=0}^{K-1} \mathbb{E} \| \Tilde{\nabla} f_i(Q_{t,\tau}) - \nabla f_i(Q_{t,\tau}) \|^2 \notag \\
\overset{(c)}{\leq} & 2N \sum_{i \in [N]} p_i^2 (\hat{r}_t^i)^2 + 6 L^2 \eta_t^2 NK \sum_{i \in [N]} p_i^2 \sum_{\tau=0}^{K-1} \mathbb{E} \| Q_{t,\tau}^{r_i} - Q_{t,\tau} \|^2 \notag \\
& + 12 \eta_t^2 NK^2 \sigma_l^2 \sum_{i \in [N]} p_i^2 \notag \\
\overset{(d)}{\leq} & 2N \sum_{i \in [N]} p_i^2 (\hat{r}_t^i)^2 + 6 L^2 \eta_t^2 NK \sum_{i \in [N]} p_i^2 \sum_{\tau=0}^{K-1} [D_0^K (\hat{r}_i)^2 \notag \\ & + 8 \frac{D_0^K - 1}{D_0 - 1} \eta_t^2 \sigma_l^2]  + 12 \eta_t^2 NK^2 \sigma_l^2 \sum_{i \in [N]} p_i^2 \notag \\
\leq & 2N \sum_{i \in [N]} p_i^2 (\hat{r}_t^i)^2 + 6 L^2 \eta_t^2 NK^2 D_0^K \sum_{i \in [N]} p_i^2 (\hat{r}_i)^2 \notag \\
& + 48 L^2 \eta_t^4 NK^2 \frac{D_0^K - 1}{D_0 - 1} \sigma_l^2 \sum_{i \in [N]} p_i^2 + 12 \eta_t^2 NK^2 \sigma_l^2 \sum_{i \in [N]} p_i^2. 
\end{align}

Inequalities (a) and (b) hold due to the Cauchy-Schwarz inequality, while inequality (c) is valid based on Assumption \ref{assumption:1} and Assumption \ref{assumption:2}. Inequality (e) follows from the results derived in Lemma \ref{lemma1}. Similar to the derivation of (\ref{eq:T_4}), handling $\mathbb{E}\langle \nabla f(V_{t+1}), Q_{t+1} - V_{t+1} \rangle$ yields the lemma's expression, completing the proof of Lemma \ref{lemma3}.
\end{proof}
\begin{lem}
We investigate the relationship between the global model $V_{t+1}$ at round $t+1$ and the global model $Q_t$ at round $t$, which is expressed as follows:
\begin{align}
\mathbb{E} f(V_{t+1}) &\leq   \mathbb{E}f(Q_t) - C \eta_t K\mathbb{E}\|\nabla f(Q_t)\|^2 + \frac{3}{2}L^2\eta_t^3K^2N*\notag \\ \quad &(\sigma_l^2 + 6K\sigma_g^2) \sum_{i\in[N]}  p_i^2 +  \frac{L \eta_t^2 K^2 N}{2} \sigma_l^2 \sum_{i \in [N]} p_i^2.
\end{align}
\label{lemma4}
\end{lem}
\begin{proof}
\begin{align}
&\mathbb{E} f(V_{t+1}) = \mathbb{E} f\left\{\sum_{i\in [N]}  p_i\left[Q_t-\sum_{\tau=0}^{K-1}\eta_t \tilde{\nabla} f_i(Q_{t,\tau}^{i})\right]\right\} \notag \\ &\leq \mathbb{E} f\left( Q_t - \sum_{i \in [N]} p_i \sum_{\tau=0}^{K-1} \eta_t \tilde{\nabla} f_i(Q_{t,\tau}^i) \right) \notag  \\
&\leq \ \mathbb{E} f(Q_t) + \mathbb{E} \left\langle \nabla f(Q_t), -\sum_{i \in [N]} p_i \sum_{\tau=0}^{K-1} \eta_t \tilde{\nabla} f_i(Q_{t,\tau}^i) \right\rangle \notag \\ & +\frac{L}{2} \mathbb{E} \left\| \sum_{i \in [N]} p_i \sum_{\tau=0}^{K-1} \eta_t \tilde{\nabla} f_i(Q_{t,\tau}^i) \right\|^2.   
\notag \\ & \leq  \mathbb{E}f(Q_t) - C \eta_t K\mathbb{E}\|\nabla f(Q_t)\|^2 +  \frac{L \eta_t^2 K^2 N}{2} \sigma_l^2 \sum_{i \in [N]} p_i^2 \notag \\ & + \frac{3}{2}L^2\eta_t^3K^2N(\sigma_l^2 + 6K\sigma_g^2) \sum_{i\in[N]}  p_i^2 . 
\end{align}
The validity of the last inequality follows directly from Lemma~\ref{lemma2}. See the convergence analysis of FedHL for a detailed derivation. This concludes the proof of Lemma~\ref{lemma4}.
\end{proof}

Next, by applying Lemmas \ref{lemma3} and \ref{lemma4}, we establish the following expression for the global models $Q_{t+1}$ and $Q_t$:
\begin{align}
\label{app:sum5.4}
 &\mathbb{E} f(Q_{t+1}) \leq \mathbb{E}f(Q_t) - C \eta_t K \mathbb{E}\big\|\nabla f(Q_t)\big\|^2 + \frac{3}{2}L^2\eta_t^3K^2N*\notag \\&\quad(\sigma_l^2 + 6K\sigma_g^2) \sum_{i\in[N]}  p_i^2 +  \frac{L \eta_t^2 K^2 N}{2} \sigma_l^2 \sum_{i \in [N]} p_i^2 + \frac{(L+1)}{2}*\notag \\ & \quad[2 N \sum_{i \in [N]} (p_i)^2(\hat{r}_t^i)^2 + 6 L^2 \eta_t^2 NK^2 D_0^K \sum_{i \in [N]} (p_i)^2 (\hat{r}_i)^2 \notag \\ & \quad+ 48 L^2 \eta_t^4NK^2 \frac{D_0^K-1}{D_0-1}\sigma_l^2\sum_{i \in [N]} p_i^2 + 12 \eta_t^2 N K^2 \sigma_l^2 \sum_{i \in [N]} p_i^2]. 
\end{align}

We sum both sides of Eqn. (\ref{app:sum5.4}) over $t = 0, \dots , T-1$, and rearrange the terms.
\begin{align}
&\sum_{t=0}^{T-1} C \eta_t K \mathbb{E}\big\|\nabla f(Q_t)\big\|^2 \leq \mathbb{E}f(Q_{0}) - \mathbb{E} f(Q_T)    \notag \\ &+ \frac{(L+1)}{2}[2 N T\sum_{i \in [N]} (p_i)^2(\hat{r}_i)^2 + 6 L^2  NK^2 D_0^K * \notag \\ & \sum_{i \in [N]} (p_i)^2 (\hat{r}_i)^2 \sum_{t=0}^{T-1}\eta_t^2+ 48 L^2 NK^2 \frac{D_0^K-1}{D_0-1}\sigma_l^2\sum_{i \in [N]} p_i^2 \sum_{t=0}^{T-1}\eta_t^4  
\notag \\  &+ 12  N K^2 \sigma_l^2 \sum_{i \in [N]} p_i^2 \sum_{t=0}^{T-1}\eta_t^2] +\frac{3}{2}L^2K^2N(\sigma_l^2 + 6K\sigma_g^2) * \notag \\ & \quad\sum_{i\in[N]}  p_i^2 \sum_{t=0}^{T-1}\eta_t^3 +  \frac{L  K^2 N}{2} \sigma_l^2 \sum_{i \in [N]} p_i^2 \sum_{t=0}^{T-1}\eta_t^2 . 
\end{align}
Let the learning rate \(\eta_t = \frac{1}{LK \sqrt{T}}\), which yields the final result.

\section{Proof of Theorem \ref{theorem:3}}
\label{Appendix:th3}

To improve the convergence bound, we focus on optimizing the terms with the slowest convergence rate, all scaled by \( \mathcal{O}(\frac{1}{\sqrt{T}}) \), from the derived upper bound. We focus on minimizing the three terms $\mathcal{L}_1(p_i) = \frac{3NL(L+1)D_0^K}{2\sqrt{T}} \sum_{i \in [N]} p_i^2 \hat{r}_i^2
$, $\mathcal{L}_2(p_i) = \frac{3N(L+1) \sigma_l^2}{L\sqrt{T}} \sum_{i \in [N]} p_i^2
$, $\mathcal{L}_3(p_i) = \frac{N \sigma_l^2}{2\sqrt{T}} \sum_{i \in [N]} p_i^2
$, scaled by \( \frac{1}{\sqrt{T}} \), as they dominate the convergence rate:
\begin{align}
    \mathcal{L}(p_i) = \frac{1}{\sqrt{T}} \Bigg[ \sum_{i \in [N]} p_i^2 \Big( \frac{3}{2}NL(L+1)D_0^K \hat{r}_i^2 + \frac{N \sigma_l^2 (7L + 6)}{2L} \Big) \Bigg].
\end{align}
Letting \(A = \frac{3}{2}NL(L+1)D_0^K\) and \(B = \frac{N \sigma_l^2 (7L + 6)}{2L}\), the optimization problem can be reformulated as:
\begin{align}
    \min_{p_i} \sum_{i \in [N]} p_i^2 \Big( A \hat{r}_i^2 + B \Big), \quad \text{s.t.} \quad \sum_{i \in [N]} p_i = 1.
\end{align}

To simplify, we divide the entire objective by \(A\), reducing it to:
\begin{align}
    \min_{p_i} \sum_{i \in [N]} p_i^2 \Big( \hat{r}_i^2 + \epsilon \Big), \quad \text{s.t.} \quad \sum_{i \in [N]} p_i = 1,
\end{align}
where \( \epsilon = \frac{B}{A} \). The parameter \(\epsilon > 0\) was introduced due to the stochastic variance of SGD and the \(L\)-smooth properties of the objective function.

\textbf{Derivation of the Optimal \(p_i\):} Using the method of Lagrange multipliers, we define the Lagrangian function as:
\begin{align}
    \mathcal{L}(p_1, \dots, p_N, \lambda) = \sum_{i \in [N]} p_i^2 \Big( \hat{r}_i^2 + \epsilon \Big) + \lambda \Big( \sum_{i \in [N]} p_i - 1 \Big).
\end{align}

Taking the partial derivative with respect to \(p_i\) and setting it to zero yields:
\begin{align}
    \frac{\partial \mathcal{L}}{\partial p_i} = 2p_i \Big( \hat{r}_i^2 + \epsilon \Big) + \lambda = 0, \quad \implies \quad p_i = -\frac{\lambda}{2(\hat{r}_i^2 + \epsilon)}.
\end{align}

Substituting the normalization constraint \( \sum_{i \in [N]} p_i = 1 \) into the equation gives:
\begin{align}
    \sum_{i \in [N]} -\frac{\lambda}{2(\hat{r}_i^2 + \epsilon)} = 1, \quad \implies \quad \lambda = -\frac{2}{\sum_{i \in [N]} \frac{1}{\hat{r}_i^2 + \epsilon}}.
\end{align}

Thus, the optimal solution for \(p_i\) is:
\begin{align}
    p_i^* = \frac{\frac{1}{\hat{r}_i^2 + \epsilon}}{\sum_{j \in [N]} \frac{1}{\hat{r}_j^2 + \epsilon}}.
\end{align}

\bibliographystyle{IEEEtran}
\bibliography{manuscript}

\end{document}